\documentclass[final]{cvpr}

\usepackage{times}
\usepackage{epsfig}
\usepackage{graphicx}
\usepackage{amsmath}
\usepackage{amssymb}
\usepackage{bbold}
\usepackage{amsthm}
\usepackage[dvipsnames]{xcolor}

\usepackage[labelformat=simple]{subcaption}
\usepackage{capt-of}
\usepackage{multirow}
\usepackage{booktabs} 

\usepackage{caption}
\usepackage{makecell}

\newcommand\m[1]{\begin{pmatrix}#1\end{pmatrix}} 
\newtheorem{theorem}{Theorem}[]
\newtheorem{lemma}[theorem]{Lemma}

\usepackage[pagebackref=true,breaklinks=true,colorlinks,bookmarks=false]{hyperref}

\newcommand\abs[1]{\left|#1\right|}

\def\confYear{CVPR 2021}

\begin{document}

\title{Transformer Interpretability Beyond Attention Visualization}

\author{Hila Chefer\textsuperscript{\rm 1} \quad Shir Gur\textsuperscript{\rm 1} \quad Lior Wolf\textsuperscript{\rm 1,2}\\
\textsuperscript{\rm 1}The School of Computer Science, Tel Aviv University\\
\textsuperscript{\rm 2}Facebook AI Research (FAIR)\\
}

\maketitle

\begin{abstract}
Self-attention techniques, and specifically Transformers, are dominating the field of text processing and are becoming increasingly popular in computer vision classification tasks. In order to visualize the parts of the image that led to a certain classification, existing methods either rely on the obtained attention maps or employ heuristic propagation along the attention graph. In this work, we propose a novel way to compute relevancy for Transformer networks. The method assigns local relevance based on the Deep Taylor Decomposition principle and then propagates these relevancy scores through the layers. This propagation involves attention layers and skip connections, which challenge existing methods. Our solution is based on a specific formulation that is shown to maintain the total relevancy across layers. We benchmark our method on very recent visual Transformer networks, as well as on a text classification problem, and demonstrate a clear advantage over the existing explainability methods. Our code is available at: \url{https://github.com/hila-chefer/Transformer-Explainability}.
\end{abstract}

\section{Introduction}

Transformers and derived methods~\cite{vaswani2017attention,devlin2018bert,liu2019roberta,radford2019language} are currently the state-of-the-art methods in almost all NLP benchmarks. The power of these methods has led to their adoption in the field of language and vision~\cite{lu2019vilbert,tan2019lxmert,su2019vl}. More recently, Transformers have become a leading tool in traditional computer vision tasks, such as object detection~\cite{carion2020end} and image recognition~\cite{chen2020generative,dosovitskiy2020image}.
The importance of Transformer networks necessitates tools for the visualization of their decision process. Such a visualization can aid in debugging the models, help verify that the models are fair and unbiased, and enable downstream tasks.

The main building block of Transformer networks are self-attention layers~\cite{parikh2016decomposable,cheng2016long}, which assign a pairwise attention value between every two tokens. In NLP, a token is typically a word or a word part. In vision, each token can be associated with a patch~\cite{dosovitskiy2020image,carion2020end}. A common practice when trying to visualize Transformer models is, therefore, to consider these attentions as a relevancy score~\cite{vaswani2017attention,xu2015show,carion2020end}.  This is usually done for a single attention layer.
Another option is to combine multiple layers. Simply averaging the attentions obtained for each token, would lead to blurring of the signal and would not consider the different roles of the layers: deeper layers are more semantic, but each token accumulates additional context each time self-attention is applied.
The rollout method~\cite{abnar2020quantifying} is an alternative, which reassigns all attention scores by considering the pairwise attentions and assuming that attentions are combined linearly into subsequent contexts. The method seems to improve results over the utilization of a single attention layer. However, as we show, by relying on simplistic assumptions, irrelevant tokens often become highlighted.

In this work, we follow the line of work that assigns relevancy and propagates it, such that the sum of relevancy is maintained throughout the layers~\cite{montavon2017explaining}. While the application of such methods to Transformers has been attempted~\cite{voita2019analyzing}, this was done in a partial way that does not propagate attention throughout all layers. 

Transformer networks heavily rely on skip connection and attention operators, both involving the mixing of two activation maps, and each leading to unique challenges. Moreover, Transformers apply non-linearities other than ReLU, which result in both positive and negative features. Because of the non-positive values, skip connections lead, if not carefully handled, to numerical instabilities. Methods such as LRP~\cite{binder2016layer} for example, tend to fail in such cases. 
Self-attention layers form a challenge since a naive propagation through these would not maintain the total amount of relevancy.

We handle these challenges by first introducing a relevancy propagation rule that is applicable to both positive and negative attributions. Second, we present a normalization term for non-parametric layers, such as ``add'' (\eg skip-connection) and matrix multiplication. Third, we integrate the attention and the relevancy scores, and combine the integrated results for multiple attention blocks.

Many of the interpretability methods used in computer vision are not class-specific in practice, \ie, return the same visualization regardless of the class one tries to visualize, even for images that contain multiple objects. The class-specific signal, especially for methods that propagate all the way to the input, is often blurred by the salient regions of the image. Some methods avoid this by not propagating to the lower layers~\cite{selvaraju2017grad}, while other methods contrast different classes to emphasize the differences~\cite{gu2018understanding}. Our method provides the class-based separation by design and it is the only Transformer visualization method, as far as we can ascertain, that presents this property.

Explainability, interpretability, and relevance are not uniformly defined in the literature~\cite{mittelstadt2019explaining}. For example, it is not clear if one would expect the resulting image to contain all of the pixels of the identified object, which would lead to better downstream tasks~\cite{li2018tell} and for favorable human impressions, or to identify the sparse image locations that cause the predicted label to dominate. While some methods offer a clear theoretical framework~\cite{lundberg2017unified}, these rely on specific assumptions and often do not lead to better performance on real data. Our approach is a mechanistic one and avoids  controversial issues. Our goal is to improve the performance on the acceptable benchmarks of the field. This goal is achieved on a diverse and complementary set of computer vision benchmarks, representing multiple approaches to explainability.

These benchmarks include image segmentation on a subset of the ImageNet dataset, as well as positive and negative perturbations on the ImageNet validation set. In NLP, we consider a public NLP explainability benchmark~\cite{deyoung2019eraser}. In this benchmark, the task is to identify the excerpt that was marked by humans as leading to a decision.

\section{Related Work}

\paragraph{Explainability in computer vision} Many methods were suggested for generating a heatmap that indicates local relevancy, given an input image and a CNN. Most of these methods belong to one of two classes: gradient methods and attribution methods. 

{\em Gradient based} methods are based on the gradients with respect to the input of each layer, as computed through backpropagation. The gradient is often multiplied by the input activations, which was first done in the Gradient*Input method~\cite{shrikumar2016not}. Integrated Gradients~\cite{sundararajan2017axiomatic} also compute the multiplication of the inputs with their derivatives. However, this computation is done on the average gradient and a linear interpolation of the input.
SmoothGrad~\cite{smilkov2017smoothgrad}, visualizes the mean gradients of the input, and performs smoothing by adding to the input image a random Gaussian noise at each iteration. The FullGrad method~\cite{srinivas2019full} offers a more complete modeling of the gradient by also considering the gradient with respect to the bias term, and not just with respect to the input.
We observe that these methods are all class-agnostic: at least in practice, similar outputs are obtained,  regardless of the class used to compute the gradient that is being propagated.

The GradCAM method~\cite{selvaraju2017grad} is a class-specific approach, which combines both the input features and the gradients of a network's layer. Being class-specific, and providing consistent results, this method is used by downstream applications, such as weakly-supervised semantic segmentation~\cite{li2018tell}. However, the method's computation is based only on the gradients of the deepest layers. The result, obtained by upsampling these low-spatial resolution layers, is coarse.

A second class of methods, the {\em Attribution propagation} methods, are justified theoretically by the Deep Taylor Decomposition (DTD) framework~\cite{montavon2017explaining}. Such methods decompose, in a recursive manner, the decision made by the network, into the contributions of the previous layers, all the way to the elements of the network's input.
The Layer-wise Relevance Propagation (LRP) method~\cite{bach2015pixel}, propagates relevance from the predicated class, backward, to the input image based on the DTD principle. This assumes that the rectified linear unit (ReLU) non-linearity is used. Since Transformers typically rely on other types of applications, our method has to apply DTD differently. Other variants of attribution methods include RAP~\cite{nam2019relative}, AGF~\cite{gur2021visualization}, DeepLIFT~\cite{shrikumar2017learning}, and DeepSHAP~\cite{lundberg2017unified}.
A disadvantage of some of these methods is the class-agnostic behavior observed in practice~\cite{iwana2019explaining}. Class-specific behavior is obtained by Contrastive-LRP (CLRP)~\cite{gu2018understanding} and Softmax-Gradient-LRP (SGLRP)~\cite{iwana2019explaining}. In both cases, the LRP propagation results of the class to be visualized are contrasted with the results of all other classes, to emphasize the differences and produce a class-dependent heatmap. Our method is class-specific by construction and not by adding additional contrasting stages.

Methods that do not fall into these two main categories include saliency based methods~\cite{dabkowski2017real,simonyan2013deep,mahendran2016visualizing,zhou2016learning,zeiler2014visualizing,zhou2018interpreting}, {Activation Maximization}~\cite{erhan2009visualizing} and Excitation Backprop~\cite{zhang2018top}.
Perturbation methods~\cite{fong2019understanding,fong2017interpretable} consider the change to the decision of the network, as small changes are applied to the input. Such methods are intuitive and applicable to black-box models (no need to inspect either the activations or the gradients). However, the process of generating the heatmap is computationally expensive. In the context of Transformers, it is not clear how to apply these correctly to discrete tokens, such as in text.
Shapley-value methods~\cite{lundberg2017unified} have a solid theoretical justification. However, such methods suffer from a large computational complexity and their accuracy is often not as high as other methods. Several variants have been proposed, which improve both aspects~\cite{chen2018lshapley}. 

\paragraph{Explainability for Transformers}

There are not many contributions that explore the field of visualization for Transformers and, as mentioned, many contributions employ the attention scores themselves. This practice ignores most of the attention components, as well as the parts of the networks that perform other types of computation. A self-attention head involves the computation of queries, keys, and values. Reducing it only to the obtained attention scores (inner products of queries and keys) is myopic. Other layers are not even considered. Our method, in contrast, propagates through all layers from the decision back to the input.

LRP was applied for Transformers based on the premise that considering mean attention heads is not optimal due to different relevance of the attention heads in each layer~\cite{voita2019analyzing}. However, this was done in a limiting way, in which no relevance scores were propagated back to the input, thus providing partial information on the relevance of each head.
We note that the relevancy scores were not directly evaluated, only used for visualization of the relative importance and for pruning less relevant attention heads. 

The main challenge in assigning attributions based on attentions is that attentions are combining non-linearly from one layer to the next. The rollout method~\cite{abnar2020quantifying} assumes that attentions are combined linearly and considers paths along the pairwise attention graph. We observe that this method often leads to an emphasis on irrelevant tokens since even average attention scores can be attenuated. The method also fails to distinguish between positive and negative contributions to the decision. Without such a distinction, one can mix between the two and obtain high relevancy scores, when the contributions should have cancelled out.
Despite these shortcomings, the method was already applied by others~\cite{dosovitskiy2020image} to obtain integrated attention maps.

Abnar et al.~\cite{abnar2020quantifying} present, in addition to rollout, a second method called attention flow. The latter considers the max-flow problem along the pair-wise attention graph. It is shown to be sometimes more correlated than the rollout method with relevance scores that are obtained by applying masking, or with gradients with respect to the input. This method is much slower and we did not evaluate it in our experiments for computational reasons.

We note this concurrent work~\cite{abnar2020quantifying} did not perform an evaluation on benchmarks (for either rollout or attention-flow) in which relevancy is assigned in a way that is independent of the BERT~\cite{devlin2018bert} network, for which the methods were employed. There was also no comparison to relevancy assignment methods, other than the raw attention scores.

\section{Method}
The method employs LRP-based relevance to compute scores for each attention head in each layer of a Transformer model~\cite{vaswani2017attention}. It then integrates these scores throughout the attention graph, by incorporating both relevancy and gradient information, in a way that iteratively removes the negative contributions. The result is a class-specific visualization for self-attention models.

\subsection{Relevance and gradients}
Let $C$ be the number of classes in the classification head, and $t \in 1\dots|C|$ the class to be visualized.  We propagate relevance and gradients with respect to class $t$, which is not necessarily the predicted class.
Following literature convention, we denote $x^{(n)}$ as the input of layer $L^{(n)}$, where $n\in [1\dots N]$ is the layer index in a network that consists of $N$ layers, $x^{(N)}$ is the input to the network, and $x^{(1)}$ is the output of the network.

Recalling the chain-rule, we propagate gradients with respect to the classifier's output $y$, at class $t$, namely $y_t$:
\begin{align}
    \label{eq:chain_rule}
    \nabla x_j^{(n)} := \frac{\partial y_t}{\partial x_j^{(n)}} &= \sum_{i} \frac{\partial y_t}{\partial x_i^{(n-1)}} \frac{\partial x_i^{(n-1)}}{\partial x_j^{(n)}}
\end{align}
where the index $j$ corresponds to elements in $x^{(n)}$, and $i$ corresponds to elements in $x^{(n-1)}$.

We denote by $L^{(n)}(\mathbf{X}, \mathbf{Y})$ the layer's operation on two tensors $\mathbf{X}$ and $\mathbf{Y}$. Typically, the two tensors are the input feature map and weights for layer $n$. Relevance propagation follows the generic Deep Taylor Decomposition~\cite{montavon2017explaining}:
\begin{align}
    \label{eq:generic_rel}
    R_j^{(n)} &= \mathcal{G}(\mathbf{X},\mathbf{Y},R^{(n-1)})\\
    \nonumber
    &= \sum_i \mathbf{X}_j\frac{\partial L^{(n)}_i(\mathbf{X}, \mathbf{Y})}{\partial \mathbf{X}_j} \frac{R^{(n-1)}_i}{L^{(n)}_{{i}}(\mathbf{X}, \mathbf{Y})}\,,
\end{align}
where, similarly to Eq.~\ref{eq:chain_rule}, the index $j$ corresponds to elements in $R^{(n)}$, and $i$ corresponds to elements in $R^{(n-1)}$.
Eq.~\ref{eq:generic_rel} satisfies the conservation rule~\cite{montavon2017explaining}, \ie:
\begin{align}
    \label{eq:conservation}
    \sum_j R^{(n)}_j = \sum_i R^{(n-1)}_i
\end{align}
{LRP~\cite{bach2015pixel} assumes ReLU non-linearity activations, resulting in non-negative feature maps, where the relevance propagation rule can be defined as follows:
\begin{align}
    \label{eq:lrp}
    R_j^{(n)} &= \mathcal{G}(x^+,w^+,R^{(n-1)}) = \sum_i \frac{x^+_j w_{ji}^+}{\sum_{{j^\prime}}x^+_{j^\prime} w_{{j^\prime}i}^+}R^{(n-1)}_i
\end{align}
where $\mathbf{X}=x$ and $\mathbf{Y}=w$ are the layer's input and weights. The superscript denotes the operation $\max(0, v)$ as $v^+$.}

{Non-linearities other that ReLU, such as GELU~\cite{hendrycks2016gaussian}, output both positive and negative values.
To address this, LRP propagation in Eq.~\ref{eq:lrp} can be modified by constructing a subset of indices $q = \{(i,j) | x_j w_{ji}\ge0 \}$, resulting in the following relevance propagation:
\begin{align}
    \label{eq:relevance}
    \nonumber
    R_j^{(n)} &= \mathcal{G}_q(x,w,q,R^{(n-1)})\\
    &= \sum_{\{{i}|({i},j)\in q\}} \frac{x_j w_{ji}}{\sum_{\{{j^\prime}|({j^\prime},i)\in q\}}x_{j^\prime} w_{{j^\prime}i}}R^{(n-1)}_i
\end{align}
In other words, we consider only the elements that have a positive weighed relevance.}

To initialize the relevance propagation, we set $R^{(0)} = \mathbb{1}_t$,
where $\mathbb{1}_t$ is a one-hot indicating the target class $t$.

\subsection{Non parametric relevance propagation:}
There are two operators in Transformer models that involve mixing of two feature map tensors (as opposed to a feature map with a learned tensor): skip connections and matrix multiplications (\eg in attention modules). The two operators require the propagation of relevance through both input tensors. Note that the two tensors may be of different shapes in the case of matrix multiplication.

Given two tensors $u$ and $v$, we compute the relevance propagation of these binary operators (\ie, operators that process two operands), as follows:
\begin{align}
    R_j^{u^{(n)}} = \mathcal{G}(u,v,R^{(n-1)}), \quad
    R_k^{v^{(n)}} = \mathcal{G}(v,u,R^{(n-1)}) \label{eq:uvr}
\end{align}
where $R_j^{u^{(n)}}$ and $R_k^{v^{(n)}}$ are the relevances for $u$ and $v$ respectively. 
These operations yield both positive and negative values.

The following lemma shows that for the case of addition, the conservation rule is preserved, \ie, 
\begin{equation}
\sum_j R_j^{u^{(n)}} + \sum_k R_k^{v^{(n)}} = \sum_i R_i^{(n-1)}
\label{eq:binarycr}.
\end{equation}
However, this is not the case for matrix multiplication.

\begin{lemma}
\label{lemma1}
Given two tensors $u$ and $v$, consider the relevances that are computed according to Eq.~\ref{eq:uvr}. Then, (i) if layer $L^{(n)}$ adds the two tensors, \ie,  $L^{(n)}(u,v) = u+v$ then the conservation rule of Eq.~\ref{eq:binarycr} is maintained. (ii) if the layer performs matrix multiplication $L^{(n)}(u,v) = uv$, then Eq.~\ref{eq:binarycr} does not hold in general.
\end{lemma}
\begin{proof}
(i) and (ii) are obtained from the output derivative of $L^{(n)}$ with respect to $\mathbf{X}$. In an add layer, $u$ and $v$ are independent of each other, while in matrix multiplication they are connected. A detailed proof of Lemma~\ref{lemma1} is available in the supplementary.
\end{proof}

When propagating relevance of skip connections, we encounter numerical instabilities. This arises despite the fact that, by the conservation rule of the addition operator, the sum of relevance scores is constant. The underlying reason is that the relevance scores tend to obtain large absolute values, due to the way they are computed (Eq.~\ref{eq:generic_rel}). To see this, consider the following example:
\begin{align}
    u = \m{e^a\\e^b} \quad , v = \m{1-e^a\\1-e^b}, \quad  R = \m{1\\1} \\
    R^u = \m{\frac{e^a}{e^a - e^a + 1}1\\\frac{e^b}{e^b - e^b + 1}1} = \m{e^a\\e^b}, \quad  R^v = \m{1-e^a\\1-e^b}
\end{align}
where $a$ and $b$ are large positive numbers. It is easy to verify that $\sum R^u + \sum R^v = {e^a + 1-e^a+e^b + 1-e^b} = \sum R$.  
As can be seen, while the conservation rule is preserved, the relevance scores of $u$ and $v$ may explode. See supplementary for a step by step computation.

To address the lack of conservation in the attention mechanism due to matrix multiplication, and the numerical issues of the skip connections, our method applies a normalization to $R_j^{u^{(n)}}$ and $R_k^{v^{(n)}}$:

\begin{align}
    \nonumber
    \bar{R}_j^{u^{(n)}} &= R_j^{u^{(n)}} \frac{ \abs{\sum_j{R_j^{u^{(n)}}}}}{\abs{\sum_j {R_j^{u^{(n)}}}} + \abs{\sum_k{R_k^{v^{(n)}}}}} \cdot \frac{\sum_i R_i^{(n-1)}}{\sum_j R_j^{u^{(n)}}}\\
    \bar{R}_k^{v^{(n)}} &= R_k^{v^{(n)}} \frac{\abs{\sum_k {R_k^{v^{(n)}}}}}{\abs{\sum_j {R_j^{u^{(n)}}}} + \abs{\sum_k {R_k^{v^{(n)}}}}} \cdot \frac{\sum_i R_i^{(n-1)}}{\sum_k R_k^{v^{(n)}}}
    \label{eq:norm} \notag
\end{align}

Following the conservation rule (Eq.~\ref{eq:conservation}), and the initial relevance, we obtain $\sum_i R_i^{(n)} = 1$ for each layer $n$.

The following lemma presents the properties of the normalized relevancy scores. 
\begin{lemma}
\label{lemma2}
The normalization technique upholds the following properties: (i) it maintains the conservation rule, i.e.: $\sum_j \bar{R}_j^{u^{(n)}} + \sum_k \bar{R}_k^{v^{(n)}} =  \sum_i R_i^{(n-1)}$, (ii) it bounds the relevance sum of each tensor such that:
\begin{align}
    0 \le \sum_j \bar{R}_j^{u^{(n)}}, \sum_k \bar{R}_k^{v^{(n)}} \le  \sum_i R_i^{(n-1)}
\end{align}
\end{lemma}
\begin{proof}
See supplementary.
\end{proof}

\subsection{Relevance and gradient diffusion}
Let $M$ be a Transformer model consisting of $B$ blocks, where each block $b$ is composed of self-attention, skip connections, and additional linear and normalization layers in a certain assembly. The model takes as an input a sequence of $s$ tokens, each of dimension $d$, with a special token for classification, commonly identified as the token \texttt{[CLS]}. $M$ outputs a classification probability vector $y$ of length $C$, computed using the classification token.
The self-attention module operates on a small sub-space $d_h$ of the embedding dimension $d$, where $h$ is the number of ``heads'', such that $hd_h = d$. The self-attention module is defined as follows:
\begin{align}
    \label{eq:softmax}
    \mathbf{A}^{(b)}& = softmax(\frac{\mathbf{Q}^{(b)}\cdot{\mathbf{K}^{(b)}}^T}{\sqrt{d_h}})\\
    \mathbf{O}^{(b)} &= \mathbf{A}^{(b)}\cdot\mathbf{V}^{(b)}
\end{align}
where $(\cdot)$ denotes matrix multiplication, $\mathbf{O}^{(b)} \in \mathbb{R}^{h \times s \times d_h}$ is the output of the attention module in block $b$, $\mathbf{Q}^{(b)}, \mathbf{K}^{(b)},\mathbf{V}^{(b)} \in \mathbb{R}^{h \times s \times d_h}$ are the query key and value inputs in block $b$, namely, different projections of an input $x^{(n)}$ for a self-attention module. $\mathbf{A}^{(b)} \in \mathbb{R}^{h \times s \times s}$ is the attention map of block $b$, where row $i$ represents the attention coefficients of each token in the input with respect to the token $i$.  
The $softmax$ in Eq.~\ref{eq:softmax} is applied, such that the sum of each row in each attention head of $\mathbf{A}^{(b)}$ is one.

Following the propagation procedure of relevance and gradients, each attention map $\mathbf{A}^{(b)}$ has its gradients $\nabla \mathbf{A}^{(b)}$, and relevance $R^{(n_b)}$, with respect to a target class $t$, 
where $n_b$ is the layer that corresponds to the $softmax$ operation in Eq.~\ref{eq:softmax} of block $b$, and $R^{(n_b)}$ is the layer's relevance.

The final output $\mathbf{C} \in \mathbb{R}^{s \times s}$ of our method is then defined by the weighted attention relevance:
\begin{align}
    \label{eq:modified_att}
    \mathbf{\bar{A}}^{(b)} &= I + \mathbb{E}_h (\nabla \mathbf{A}^{(b)} \odot R^{(n_b)})^+\\
    \mathbf{C} &= \mathbf{\bar{A}}^{(1)} \cdot \mathbf{\bar{A}}^{(2)} \cdot \ldots \cdot \mathbf{\bar{A}}^{(B)}
    \label{eq:modified_rollout}
\end{align}
where $\odot$ is the Hadamard product, and $\mathbb{E}_h$ is the mean across the ``heads'' dimension. In order to compute the weighted attention relevance, we consider only the positive values of the gradients-relevance multiplication, resembling positive relevance. To account for the skip connections in the Transformer block, we add the identity matrix to avoid self inhibition for each token.

For comparison, using the same notation, the rollout~\cite{abnar2020quantifying} method is given by:
\begin{align}
    \mathbf{\hat{A}}^{(b)} &= I + \mathbb{E}_h \mathbf{A}^{(b)}\\
    \label{eq:rollout}
    \text{rollout} &= \mathbf{\hat{A}}^{(1)} \cdot \mathbf{\hat{A}}^{(2)} \cdot \ldots \cdot \mathbf{\hat{A}}^{(B)}
\end{align}
We can observe that the result of rollout is fixed given an input sample, regardless of the target class to be visualized. In addition, it does not consider any signal, except for the pairwise attention scores.

\subsection{Obtaining the image relevance map}
\label{sec:att_to_rel}
The resulting explanation of our method is a matrix $\mathbf{C}$ of size $s \times s$, where $s$ represents the sequence length of the input fed to the Transformer. Each row corresponds to a relevance map for each token given the other tokens - following the attention computation convention in Eq.~\ref{eq:modified_rollout},~\ref{eq:softmax}. Since this work focuses on classification models, only the \texttt{[CLS]} token, which encapsulates the explanation of the classification, is considered. The relevance map is, therefore, derived from the row $\mathbf{C}_{\texttt{[CLS]}} \in \mathbb{R}^s$ that corresponds to the \texttt{[CLS]} token.
This row contains a score evaluating each token's influence on the classification token.

We consider only the tokens that correspond to the {actual input}, without special tokens, such as the {\texttt{[CLS]} token} and other separators. In vision models, such as ViT~\cite{dosovitskiy2020image}, the content tokens represent image patches. To obtain the final relevance map, we reshape the sequence to the patches grid size, \eg for a square image, the patch grid size is $\sqrt{s-1} \times \sqrt{s-1}$. This map is upsampled back to the size of the original image using bilinear interpolation.

\begin{figure}[t]
    \centering
    \includegraphics[width=\linewidth]{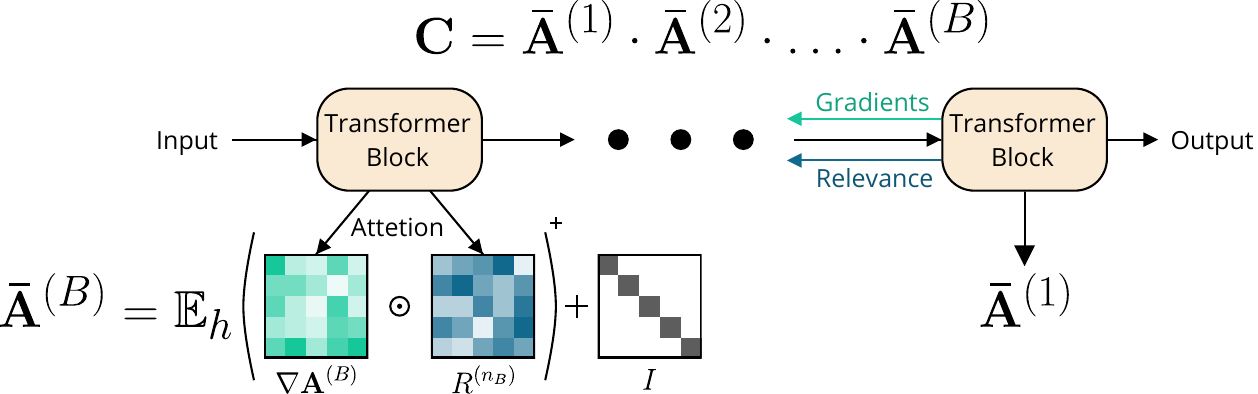}
    \caption{{Illustration of our method. Gradients and relevancies are 
    propagated through the network, and integrated to produce the final relevancy maps,
    as described in Eq.~\ref{eq:modified_att},~\ref{eq:modified_rollout}.}}
    \label{fig:method}
\end{figure}

\section{Experiments}
For the linguistic classification task, we experiment with the BERT-base~\cite{devlin2018bert} model as our classifier, assuming a maximum of 512 tokens, and a classification token \texttt{[CLS]} that is used as the input to the classification head.

For the visual classification task, we experiment with the pretrained ViT-base~\cite{dosovitskiy2020image} model, which consists of a BERT-like model. The input is a sequence of all non-overlapping patches of size $16 \times 16$ of the input image, followed by flattening and linear layers, to produce a sequence of vectors. Similar to BERT, a classification token \texttt{[CLS]} is appended at the beginning of the sequence and used for classification.

The {\bf baselines} are divided into three classes: attention-maps, relevance, and gradient-based methods. Each has different properties and assumptions over the architecture and propagation of information in the network. To best reflect the performance of different baselines, we focus on methods that are both common in the explainability literature, and applicable to the extensive tests we report in this section, \eg Black-box methods, such as Perturbation and Shapely based methods, are computationally too expensive and inherently different from the proposed method. We briefly describe each baseline in the following section and the different experiments for each domain.

The attention-map baselines include rollout~\cite{abnar2020quantifying}, following Eq.~\ref{eq:rollout}, which produces an explanation that takes into account all the attention-maps computed along 
the forward-pass. A more straightforward method is raw attention, \ie using the attention map of block $1$ to extract the relevance scores. These methods are class-agnostic by definition.

Unlike attention-map based methods, the relevance propagation methods consider the information flow through the entire network, and not just the attention maps. These baselines include Eq.~\ref{eq:lrp} and the partial application of LRP that follows~\cite{voita2019analyzing}.
As we show in our experiments, the different variants of the LRP method are practically class-agnostic, meaning the visualization remains approximately the same for different target classes.

A common class-specific explanation method is GradCAM~\cite{selvaraju2017grad}, which computes a weighted gradient-feature-map to the last convolution layer in a CNN model. The best way we found to apply GradCAM was to treat the last attention layer's \texttt{[CLS]} token as the designated feature map, without considering the \texttt{[CLS]} token itself.
We note that the last output of a Transformer model (before the classification head), is a tensor $v \in \mathbb{R}^{s \times d}$, where the first dimension relates to different input tokens, and only the \texttt{[CLS]} token is fed to the classification head. Thus, performing GradCAM on $v$ will impose a sparse gradients tensor $\nabla v$, with zeros for all tokens, except \texttt{[CLS]}.

\begin{figure*}[t]
    \setlength{\tabcolsep}{1pt} 
    \renewcommand{\arraystretch}{1} 
    \begin{center}
    \begin{tabular*}{\linewidth}{@{\extracolsep{\fill}}ccccccc}
    Input &rollout~\cite{abnar2020quantifying} & raw-attention & GradCAM~\cite{selvaraju2017grad} & LRP~\cite{binder2016layer} & partial LRP~\cite{voita2019analyzing} & Ours\\
    \includegraphics[width=0.12\linewidth]{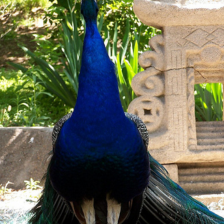} &
    \includegraphics[width=0.12\linewidth]{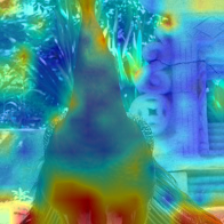} &
    \includegraphics[width=0.12\linewidth]{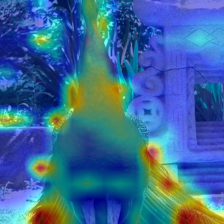} &
    \includegraphics[width=0.12\linewidth]{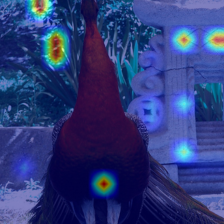} &
    \includegraphics[width=0.12\linewidth]{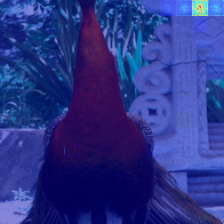} &
    \includegraphics[width=0.12\linewidth]{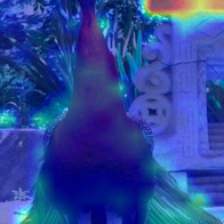} &
    \includegraphics[width=0.12\linewidth]{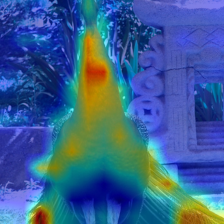}\\
    \includegraphics[width=0.12\linewidth]{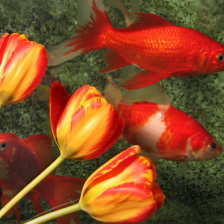} &
    \includegraphics[width=0.12\linewidth]{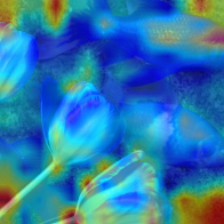} &
    \includegraphics[width=0.12\linewidth]{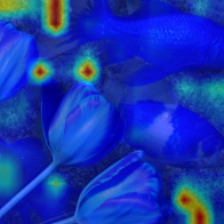} &
    \includegraphics[width=0.12\linewidth]{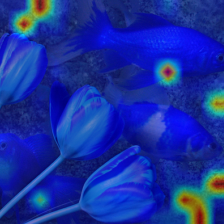} &
    \includegraphics[width=0.12\linewidth]{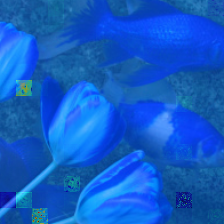} &
    \includegraphics[width=0.12\linewidth]{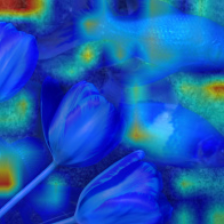} &
    \includegraphics[width=0.12\linewidth]{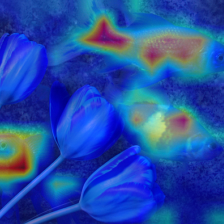}\\
    \includegraphics[width=0.12\linewidth]{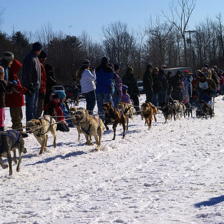} &
    \includegraphics[width=0.12\linewidth]{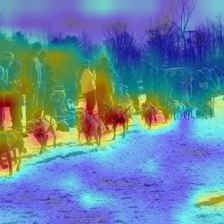} &
    \includegraphics[width=0.12\linewidth]{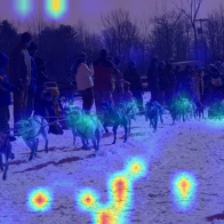} &
    \includegraphics[width=0.12\linewidth]{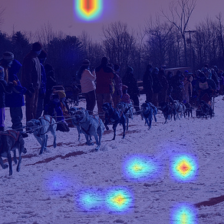} &
    \includegraphics[width=0.12\linewidth]{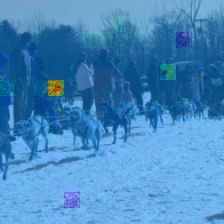} &
    \includegraphics[width=0.12\linewidth]{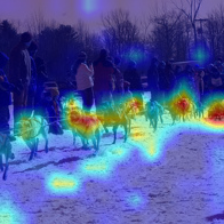} &
    \includegraphics[width=0.12\linewidth]{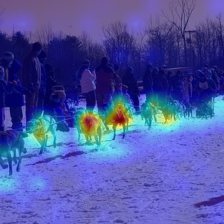}\\
    \includegraphics[width=0.12\linewidth]{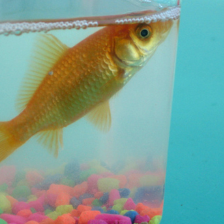} &
    \includegraphics[width=0.12\linewidth]{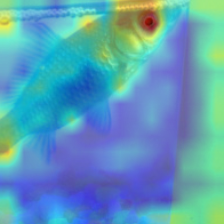} &
    \includegraphics[width=0.12\linewidth]{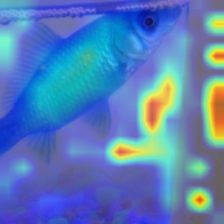} &
    \includegraphics[width=0.12\linewidth]{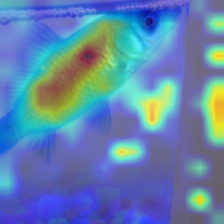} &
    \includegraphics[width=0.12\linewidth]{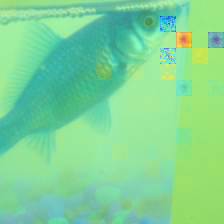} &
    \includegraphics[width=0.12\linewidth]{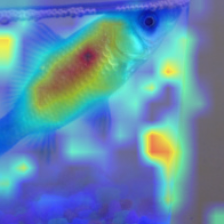} &
    \includegraphics[width=0.12\linewidth]{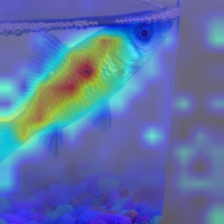}
    \end{tabular*}
    \caption{Sample results. As can be seen, our method produces more accurate visualizations.}
    \label{fig:qualitative}
    \end{center}
    \setlength{\tabcolsep}{1pt} 
    \renewcommand{\arraystretch}{1} 
    \begin{center}
    \begin{tabular*}{\linewidth}{@{\extracolsep{\fill}}ccccccc}
    Input &rollout~\cite{abnar2020quantifying} & raw-attention & GradCAM~\cite{selvaraju2017grad} & LRP~\cite{binder2016layer} & partial LRP~\cite{voita2019analyzing} & Ours\\
    \raisebox{13mm}{\multirow{2}{*}{\makecell{Dog $\rightarrow$\\\includegraphics[width=0.13\textwidth]{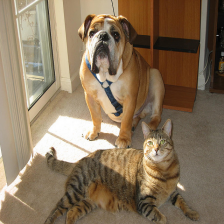}\\ Cat $\rightarrow$}}} &
    \includegraphics[width=0.12\linewidth]{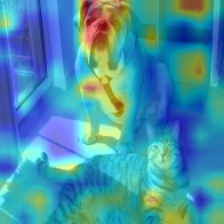} &
    \includegraphics[width=0.12\linewidth]{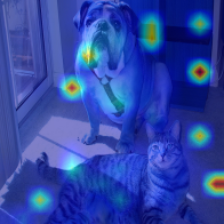} &
    \includegraphics[width=0.12\linewidth]{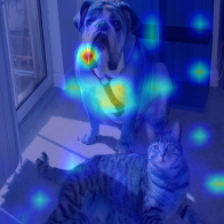} &
    \includegraphics[width=0.12\linewidth]{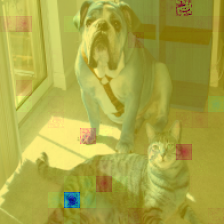} &
    \includegraphics[width=0.12\linewidth]{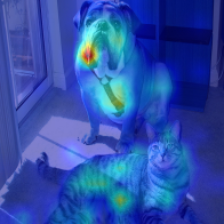} &
    \includegraphics[width=0.12\linewidth]{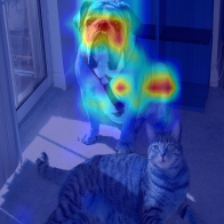}\\
    &
    \includegraphics[width=0.12\linewidth]{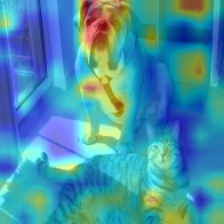} &
    \includegraphics[width=0.12\linewidth]{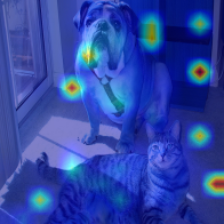} &
    \includegraphics[width=0.12\linewidth]{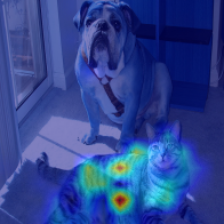} &
    \includegraphics[width=0.12\linewidth]{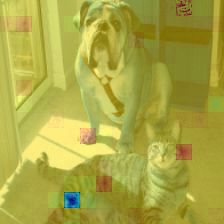} &
    \includegraphics[width=0.12\linewidth]{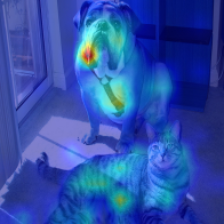} &
    \includegraphics[width=0.12\linewidth]{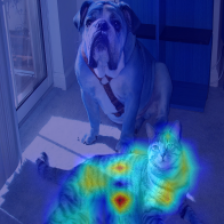}\\
    
    \raisebox{13mm}{\multirow{2}{*}{\makecell{Elephant $\rightarrow$\\\includegraphics[width=0.13\textwidth]{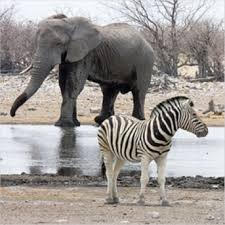}\\ Zebra $\rightarrow$}}} &
    \includegraphics[width=0.12\linewidth]{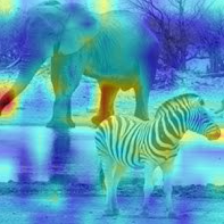} &
    \includegraphics[width=0.12\linewidth]{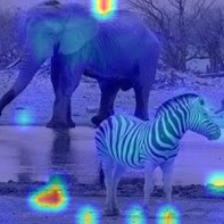} &
    \includegraphics[width=0.12\linewidth]{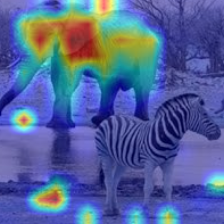} &
    \includegraphics[width=0.12\linewidth]{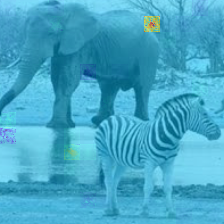} &
    \includegraphics[width=0.12\linewidth]{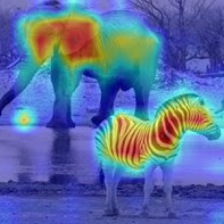} &
    \includegraphics[width=0.12\linewidth]{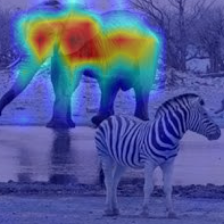}\\
    &
    \includegraphics[width=0.12\linewidth]{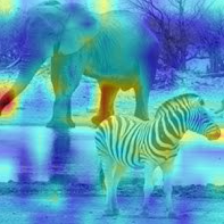} &
    \includegraphics[width=0.12\linewidth]{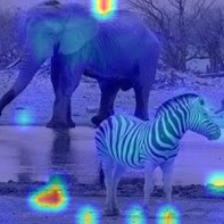} &
    \includegraphics[width=0.12\linewidth]{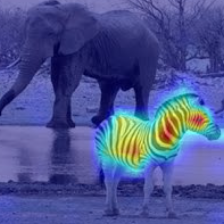} &
    \includegraphics[width=0.12\linewidth]{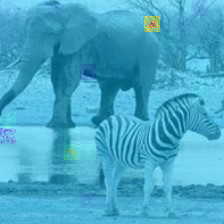} &
    \includegraphics[width=0.12\linewidth]{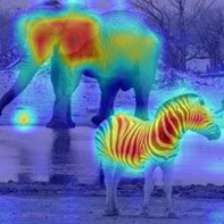} &
    \includegraphics[width=0.12\linewidth]{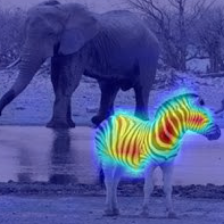}
    \end{tabular*}
    \caption{Class-specific visualizations. For each image we present results for two different classes. GradCam is the only method to generate different maps. However, its results are not convincing.}
    \label{fig:cls_spes}
    \end{center}
\end{figure*}

\smallskip
\noindent\textbf{Evaluation settings}
For the visual domain, we follow the convention of reporting results for negative and positive perturbations, as well as showing results for segmentation, which can be seen as a general case of ''The Pointing-Game''~\cite{hooker2019benchmark}. The dataset used is the validation set of ImageNet~\cite{russakovsky2015ImageNet} (ILSVRC) 2012, consisting of 50K images from 1000 classes, and an annotated subset of ImageNet called ImageNet-Segmentation~\cite{guillaumin2014ImageNet}, containing 4,276 images from 445 categories. For the linguistic domain, we follow ERASER~\cite{deyoung2019eraser} and evaluate the reasoning for the  Movies Reviews~\cite{zaidan2008modeling} dataset, which consists of 1600/200/200 reviews for train/val/test. This task is a binary sentiment analysis task. Providing explanations for question answering and entailment tasks of the other datasets in ERASER, which require input sizes of more than 512 tokens (the limit of our BERT model), is left for future work.

The positive and negative perturbation tests follow a two-stage setting. First, a pre-trained network is used for extracting visualizations for the validation set of ImageNet. Second, we gradually mask out the pixels of the input image and measure the mean top-1 accuracy of the network. In positive perturbation, pixels are masked from the highest relevance to the lowest, while in the negative version, from lowest to highest. In positive perturbation, one expects to see a steep decrease in performance, which indicates that the masked pixels are important to the classification score. In negative perturbation, a good explanation would maintain the accuracy of the model, while removing pixels that are not related to the class. In both cases, we measure the area-under-the-curve (AUC), for erasing between $10\%-90\%$ of the pixels.

The two tests can be applied to the predicted or the ground-truth class. Class-specific methods are expected to gain performance in the latter case, while class-agnostic methods would present similar performance in both tests.

The segmentation tests consider each visualization as a soft-segmentation of the image, and compare it to the ground truth segmentation of the ImageNet-Segmentation dataset. Performance is measured by (i) pixel-accuracy, obtained after thresholding each visualization by the mean value, (ii) mean-intersection-over-union (mIoU), and (iii) mean-Average-Precision (mAP), which uses the soft-segmentation to obtain a score that is threshold-agnostic.

The NLP benchmark follows the evaluation setting of ERASER~\cite{deyoung2019eraser} for rationales extraction, where the goal is to extract parts of the input that support the (ground truth) classification. The BERT model is first fine-tuned on the training set of the Movie Reviews Dataset and the various evaluation methods are applied to its results on the test set. We report the token-F1 score, which is best suited for per-token explanation (in contrast to explanations that extract an excerpt).  
To best illustrate the performance of each method, we consider a token to be part of the ``rationale'' if it is part of the top-k tokens, and show results for $k=10 \dots 80$ in steps of $10$ tokens. This way, we do not employ thresholding that may benefit some methods over others.

\subsection{Results}

\begin{table*}[t]
    \centering
    \begin{tabular*}{\linewidth}{@{\extracolsep{\fill}}llcccccc}
        \toprule
        &&rollout~\cite{abnar2020quantifying} & raw attention & GradCAM~\cite{selvaraju2017grad} & LRP~\cite{binder2016layer} & partial LRP~\cite{voita2019analyzing} & Ours\\
        \midrule
        \multirow{2}{*}{Negative} &Predicted & 53.1 & 45.55 & 41.52 & 43.49 & 50.49 & \textbf{54.16}\\
        &Target & - & - & 42.02 & 43.49 & 50.49 & \textbf{55.04} \\
        \midrule
        \multirow{2}{*}{Positive} &Predicted & 20.05 & 23.99 & 34.06 & 41.94 & 19.64 & \textbf{17.03}\\
        &Target & - & - & 33.56 & 41.93 & 19.64 & \textbf{16.04}\\
        \bottomrule
    \end{tabular*}
    \caption{Positive and Negative perturbation AUC results (percents) for the predicted and target classes, on the ImageNet~\cite{russakovsky2015ImageNet} validation set. For positive perturbation lower is better, and for negative perturbation higher is better.}
    \label{tab:perturbations}
    \medskip
    \begin{tabular*}{\linewidth}{@{\extracolsep{\fill}}lcccccc}
        \toprule
        &rollout~\cite{abnar2020quantifying} & raw attention & GradCAM~\cite{selvaraju2017grad} & LRP~\cite{binder2016layer} & partial LRP~\cite{voita2019analyzing} & Ours\\
        \midrule
        pixel accuracy & 73.54 & 67.84 & 64.44 & 51.09 & 76.31 & \textbf{79.70}\\
        mAP & 84.76 & 80.24 & 71.60 & 55.68 & 84.67 & \textbf{86.03}\\
        mIoU & 55.42 & 46.37 & 40.82 & 32.89 & 57.94 & \textbf{61.95}\\
        \bottomrule
    \end{tabular*}
    \caption{Segmentation performance on the ImageNet-segmentation~\cite{guillaumin2014ImageNet} dataset (percent). Higher is better.}
    \label{tab:segmentation}
\end{table*}

\begin{figure}[t]
    \centering
    \includegraphics[width=0.92\linewidth]{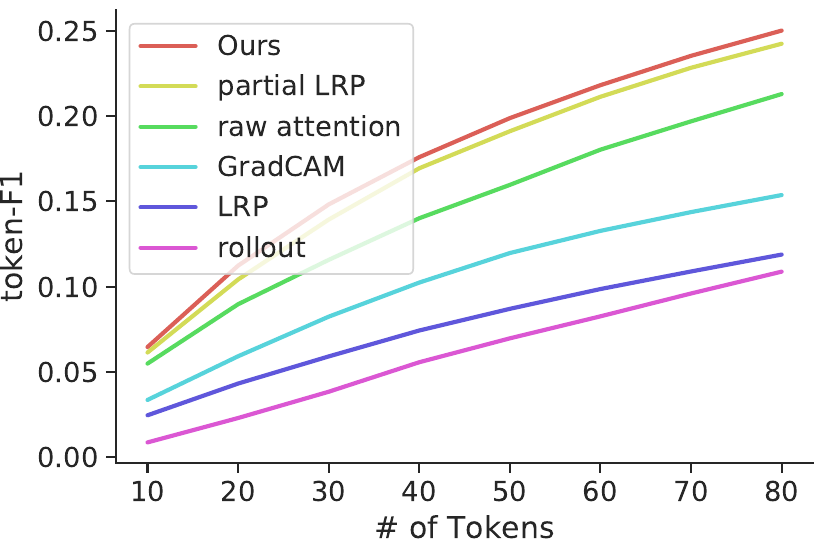}
    \caption{token-F1 scores on the Movie Reviews reasoning task.}
    \label{fig:nlp_f1}
    \medskip
    \begin{tabular*}{\linewidth}{@{\extracolsep{\fill}}lc@{~~}c@{~~}c@{~~}c@{~~}c}
        \toprule
        & \multicolumn{3}{c}{Segmentation} & \multicolumn{2}{c}{Perturbations}\\
        \cmidrule(lr){2-4}
        \cmidrule(lr){5-6}
        & Acc. & mAP & mIoU & Pos. & Neg.\\
        \midrule
        Ours w/o $\nabla \mathbf{A}^{(b)}$ & 77.66 & 85.66 & 59.88 & 18.23 & 52.88  \\
        $\nabla \mathbf{A}^{(1)}\mathbf{R}^{(n_1)}$ & 78.32 & 85.25 & 59.93& 18.01 & 52.43 \\
        $\nabla \mathbf{A}^{(B-1)}\mathbf{R}^{(n_{B-1})}$ & 60.30 & 73.63 & 39.06& 27.33 & 37.42\\
        \textbf{Ours} & \textbf{79.70} & \textbf{86.03} & \textbf{61.95}& \textbf{17.03} & \textbf{54.16}\\
        \bottomrule
    \end{tabular*}
    \captionof{table}{Performance of different variants of our method.}
    \label{tab:ablation}
\end{figure}

\noindent\textbf{Qualitative evaluation} Fig.\ref{fig:qualitative} presents a visual comparison between our method and the various baselines. As can be seen, the baseline methods produce inconsistent performance, while our method results in a much clearer and consistent visualization. 

In order to show that our method is class-specific, we show in Fig.~\ref{fig:cls_spes} images with two objects, each from a different class. As can be seen, all methods, except GradCAM, produce similar visualization for each class, while our method provides two different and accurate visualizations.

\smallskip
\noindent\textbf{Perturbation tests}
Tab.~\ref{tab:perturbations} presents the AUC obtained for both negative and positive perturbation tests, for both the predicted and the target class. As can be seen, our method achieves better performance by a large margin in both tests.
Notice that because rollout and raw attention produce constant visualization given an input image, we omit their scores in the target-class test.

\smallskip
\noindent\textbf{Segmentation}
The segmentation metrics (pixel-accuracy, mAP, and mIoU) on ImageNet-segmentation are shown in Tab.~\ref{tab:segmentation}. 
As can be seen, our method outperforms all baselines by a significant margin.

\smallskip
\noindent\textbf{Language reasoning}
Fig.~\ref{fig:nlp_f1} depicts the performance on the Movie Reviews ``rationales'' experiment, evaluating for top-K tokens, ranging from $10$ to $80$. As can be seen, while all methods benefit from increasing the amount tokens, our method consistently outperforms the baselines. See supplementary for a depiction of the obtained visualization.

\smallskip
\noindent{\bf Ablation study.} We consider three variants of our method and present their performance on the segmentation and predicted class perturbation experiments. {(i) Ours w/o $\nabla \mathbf{A}^{(b)}$, which modifies Eq.~\ref{eq:modified_att} s.t. we use $\mathbf{A}^{(b)}$ instead of $\nabla \mathbf{A}^{(b)}$, (ii) $\nabla \mathbf{A}^{(1)}\mathbf{R}^{(n_1)}$, \ie disregarding rollout in Eq.~\ref{eq:modified_rollout}, and using our method only on block $1$, which is the block closest to the output, 
and (iii) $\nabla \mathbf{A}^{(B-1)}\mathbf{R}^{(n_{B-1})}$ which similar to (ii), only for block $B-1$ which is closer to the input.}

{As can be seen in Tab.~\ref{tab:ablation} the ablation $\nabla \mathbf{A}^{(1)}\mathbf{R}^{(n_1)}$ in which one removes the rollout component, \ie, Eq.~\ref{eq:modified_rollout}, while keeping the relevance and gradient integration, and only considering the last attention layer, leads to a moderate drop in performance. Out of the two single block visualizations ((ii), and (iii)), the combined attention gradient and relevancy at the $b=1$ block, which is the closest to the output, is more informative than the block closest to the input. This is the same block that is being used for the raw-attention, partial LRP, and the GradCAM methods. The ablation that considers only this block outperforms these methods, indicating that the advantage of our method stems mostly from the combination of relevancy as we compute it and attention-map gradients.}


\section{Conclusions}

The self-attention mechanism links each of the tokens to the \texttt{[CLS]} token. The strength of this attention link can be intuitively considered as an indicator of the contribution of each token to the classification. While this is intuitive, given the term ``attention'', the attention values reflect only one aspect of the Transformer network or even of the self-attention head. As we demonstrate, both when using a fine-tuned BERT model for NLP and with the ViT model, attentions lead to fragmented and non-competitive explanations.

Despite this shortcoming and the importance of Transformer models, the literature with regards to interpretability of Transformers is sparse. In comparison to CNNs, multiple factors prevent methods developed for other forms of neural networks (not including the slower black-box methods) from being applied. These include the use of non-positive activation functions, the frequent use of skip connections, and the challenge of modeling the matrix multiplication that is used in self-attention.

Our method provides specific solutions to each of these challenges and obtains state-of-the-art results when compared to the methods of the Transformer literature, the LRP method, and the GradCam method, which can be applied directly to Transformers. 

\section*{Acknowledgment}
This project has received funding from the European Research Council (ERC) under the European Unions Horizon 2020 research and innovation programme (grant ERC CoG 725974). 
The contribution of the first author is part of a Master thesis research conducted at Tel Aviv University.

{\small
\bibliographystyle{ieee_fullname}
\bibliography{explainability}

\begin{thebibliography}{10}\itemsep=-1pt

\bibitem{abnar2020quantifying}
Samira Abnar and Willem Zuidema.
\newblock Quantifying attention flow in transformers.
\newblock {\em arXiv preprint arXiv:2005.00928}, 2020.

\bibitem{bach2015pixel}
Sebastian Bach, Alexander Binder, Gr{\'e}goire Montavon, Frederick Klauschen,
  Klaus-Robert M{\"u}ller, and Wojciech Samek.
\newblock On pixel-wise explanations for non-linear classifier decisions by
  layer-wise relevance propagation.
\newblock {\em PloS one}, 10(7):e0130140, 2015.

\bibitem{binder2016layer}
Alexander Binder, Gr{\'e}goire Montavon, Sebastian Lapuschkin, Klaus-Robert
  M{\"u}ller, and Wojciech Samek.
\newblock Layer-wise relevance propagation for neural networks with local
  renormalization layers.
\newblock In {\em International Conference on Artificial Neural Networks},
  pages 63--71. Springer, 2016.

\bibitem{carion2020end}
Nicolas Carion, Francisco Massa, Gabriel Synnaeve, Nicolas Usunier, Alexander
  Kirillov, and Sergey Zagoruyko.
\newblock End-to-end object detection with transformers.
\newblock {\em arXiv preprint arXiv:2005.12872}, 2020.

\bibitem{chen2018lshapley}
Jianbo Chen, Le Song, Martin~J. Wainwright, and Michael~I. Jordan.
\newblock L-shapley and c-shapley: Efficient model interpretation for
  structured data.
\newblock In {\em International Conference on Learning Representations}, 2019.

\bibitem{chen2020generative}
Mark Chen, Alec Radford, Rewon Child, Jeff Wu, Heewoo Jun, Prafulla Dhariwal,
  David Luan, and Ilya Sutskever.
\newblock Generative pretraining from pixels.
\newblock In {\em Proceedings of the 37th International Conference on Machine
  Learning}, volume~1, 2020.

\bibitem{cheng2016long}
Jianpeng Cheng, Li Dong, and Mirella Lapata.
\newblock Long short-term memory-networks for machine reading.
\newblock In {\em Proceedings of the 2016 Conference on Empirical Methods in
  Natural Language Processing}, pages 551--561, 2016.

\bibitem{dabkowski2017real}
Piotr Dabkowski and Yarin Gal.
\newblock Real time image saliency for black box classifiers.
\newblock In {\em Advances in Neural Information Processing Systems}, pages
  6970--6979, 2017.

\bibitem{devlin2018bert}
Jacob Devlin, Ming-Wei Chang, Kenton Lee, and Kristina Toutanova.
\newblock Bert: Pre-training of deep bidirectional transformers for language
  understanding.
\newblock {\em arXiv preprint arXiv:1810.04805}, 2018.

\bibitem{deyoung2019eraser}
Jay DeYoung, Sarthak Jain, Nazneen~Fatema Rajani, Eric Lehman, Caiming Xiong,
  Richard Socher, and Byron~C Wallace.
\newblock Eraser: A benchmark to evaluate rationalized nlp models.
\newblock {\em arXiv preprint arXiv:1911.03429}, 2019.

\bibitem{dosovitskiy2020image}
Alexey Dosovitskiy, Lucas Beyer, Alexander Kolesnikov, Dirk Weissenborn,
  Xiaohua Zhai, Thomas Unterthiner, Mostafa Dehghani, Matthias Minderer, Georg
  Heigold, Sylvain Gelly, et~al.
\newblock An image is worth 16x16 words: Transformers for image recognition at
  scale.
\newblock {\em arXiv preprint arXiv:2010.11929}, 2020.

\bibitem{erhan2009visualizing}
Dumitru Erhan, Yoshua Bengio, Aaron Courville, and Pascal Vincent.
\newblock Visualizing higher-layer features of a deep network.
\newblock {\em University of Montreal}, 1341(3):1, 2009.

\bibitem{fong2019understanding}
Ruth Fong, Mandela Patrick, and Andrea Vedaldi.
\newblock Understanding deep networks via extremal perturbations and smooth
  masks.
\newblock In {\em Proceedings of the IEEE International Conference on Computer
  Vision}, pages 2950--2958, 2019.

\bibitem{fong2017interpretable}
Ruth~C Fong and Andrea Vedaldi.
\newblock Interpretable explanations of black boxes by meaningful perturbation.
\newblock In {\em Proceedings of the IEEE International Conference on Computer
  Vision}, pages 3429--3437, 2017.

\bibitem{gu2018understanding}
Jindong Gu, Yinchong Yang, and Volker Tresp.
\newblock Understanding individual decisions of cnns via contrastive
  backpropagation.
\newblock In {\em Asian Conference on Computer Vision}, pages 119--134.
  Springer, 2018.

\bibitem{guillaumin2014ImageNet}
Matthieu Guillaumin, Daniel K{\"u}ttel, and Vittorio Ferrari.
\newblock Imagenet auto-annotation with segmentation propagation.
\newblock {\em International Journal of Computer Vision}, 110(3):328--348,
  2014.

\bibitem{gur2021visualization}
Shir Gur, Ameen Ali, and Lior Wolf.
\newblock Visualization of supervised and self-supervised neural networks via
  attribution guided factorization.
\newblock In {\em AAAI}, 2021.

\bibitem{hendrycks2016gaussian}
Dan Hendrycks and Kevin Gimpel.
\newblock Gaussian error linear units (gelus).
\newblock {\em arXiv preprint arXiv:1606.08415}, 2016.

\bibitem{hooker2019benchmark}
Sara Hooker, Dumitru Erhan, Pieter-Jan Kindermans, and Been Kim.
\newblock A benchmark for interpretability methods in deep neural networks.
\newblock In {\em Advances in Neural Information Processing Systems}, pages
  9737--9748, 2019.

\bibitem{iwana2019explaining}
Brian~Kenji Iwana, Ryohei Kuroki, and Seiichi Uchida.
\newblock Explaining convolutional neural networks using softmax gradient
  layer-wise relevance propagation.
\newblock {\em arXiv preprint arXiv:1908.04351}, 2019.

\bibitem{li2018tell}
Kunpeng Li, Ziyan Wu, Kuan-Chuan Peng, Jan Ernst, and Yun Fu.
\newblock Tell me where to look: Guided attention inference network.
\newblock In {\em Proceedings of the IEEE Conference on Computer Vision and
  Pattern Recognition}, pages 9215--9223, 2018.

\bibitem{liu2019roberta}
Yinhan Liu, Myle Ott, Naman Goyal, Jingfei Du, Mandar Joshi, Danqi Chen, Omer
  Levy, Mike Lewis, Luke Zettlemoyer, and Veselin Stoyanov.
\newblock {RoBERTa}: A robustly optimized bert pretraining approach.
\newblock {\em arXiv preprint arXiv:1907.11692}, 2019.

\bibitem{lu2019vilbert}
Jiasen Lu, Dhruv Batra, Devi Parikh, and Stefan Lee.
\newblock Vilbert: Pretraining task-agnostic visiolinguistic representations
  for vision-and-language tasks.
\newblock In {\em Advances in Neural Information Processing Systems}, pages
  13--23, 2019.

\bibitem{lundberg2017unified}
Scott~M Lundberg and Su-In Lee.
\newblock A unified approach to interpreting model predictions.
\newblock In {\em Advances in Neural Information Processing Systems}, pages
  4765--4774, 2017.

\bibitem{mahendran2016visualizing}
Aravindh Mahendran and Andrea Vedaldi.
\newblock Visualizing deep convolutional neural networks using natural
  pre-images.
\newblock {\em International Journal of Computer Vision}, 120(3):233--255,
  2016.

\bibitem{mittelstadt2019explaining}
Brent Mittelstadt, Chris Russell, and Sandra Wachter.
\newblock Explaining explanations in ai.
\newblock In {\em Proceedings of the conference on fairness, accountability,
  and transparency}, pages 279--288, 2019.

\bibitem{montavon2017explaining}
Gr{\'e}goire Montavon, Sebastian Lapuschkin, Alexander Binder, Wojciech Samek,
  and Klaus-Robert M{\"u}ller.
\newblock Explaining nonlinear classification decisions with deep taylor
  decomposition.
\newblock {\em Pattern Recognition}, 65:211--222, 2017.

\bibitem{nam2019relative}
Woo-Jeoung Nam, Shir Gur, Jaesik Choi, Lior Wolf, and Seong-Whan Lee.
\newblock Relative attributing propagation: Interpreting the comparative
  contributions of individual units in deep neural networks.
\newblock {\em arXiv preprint arXiv:1904.00605}, 2019.

\bibitem{parikh2016decomposable}
Ankur Parikh, Oscar T{\"a}ckstr{\"o}m, Dipanjan Das, and Jakob Uszkoreit.
\newblock A decomposable attention model for natural language inference.
\newblock In {\em Proceedings of the 2016 Conference on Empirical Methods in
  Natural Language Processing}, pages 2249--2255, 2016.

\bibitem{radford2019language}
Alec Radford, Jeff Wu, Rewon Child, David Luan, Dario Amodei, and Ilya
  Sutskever.
\newblock Language models are unsupervised multitask learners.
\newblock 2019.

\bibitem{russakovsky2015ImageNet}
Olga Russakovsky, Jia Deng, Hao Su, Jonathan Krause, Sanjeev Satheesh, Sean Ma,
  Zhiheng Huang, Andrej Karpathy, Aditya Khosla, Michael Bernstein, et~al.
\newblock Imagenet large scale visual recognition challenge.
\newblock {\em International journal of computer vision}, 115(3):211--252,
  2015.

\bibitem{selvaraju2017grad}
Ramprasaath~R Selvaraju, Michael Cogswell, Abhishek Das, Ramakrishna Vedantam,
  Devi Parikh, and Dhruv Batra.
\newblock Grad-cam: Visual explanations from deep networks via gradient-based
  localization.
\newblock In {\em Proceedings of the IEEE international conference on computer
  vision}, pages 618--626, 2017.

\bibitem{shrikumar2017learning}
Avanti Shrikumar, Peyton Greenside, and Anshul Kundaje.
\newblock Learning important features through propagating activation
  differences.
\newblock In {\em Proceedings of the 34th International Conference on Machine
  Learning-Volume 70}, pages 3145--3153. JMLR. org, 2017.

\bibitem{shrikumar2016not}
Avanti Shrikumar, Peyton Greenside, Anna Shcherbina, and Anshul Kundaje.
\newblock Not just a black box: Learning important features through propagating
  activation differences.
\newblock {\em arXiv preprint arXiv:1605.01713}, 2016.

\bibitem{simonyan2013deep}
Karen Simonyan, Andrea Vedaldi, and Andrew Zisserman.
\newblock Deep inside convolutional networks: Visualising image classification
  models and saliency maps.
\newblock {\em arXiv preprint arXiv:1312.6034}, 2013.

\bibitem{smilkov2017smoothgrad}
Daniel Smilkov, Nikhil Thorat, Been Kim, Fernanda Vi{\'e}gas, and Martin
  Wattenberg.
\newblock Smoothgrad: removing noise by adding noise.
\newblock {\em arXiv preprint arXiv:1706.03825}, 2017.

\bibitem{srinivas2019full}
Suraj Srinivas and Fran{\c{c}}ois Fleuret.
\newblock Full-gradient representation for neural network visualization.
\newblock In {\em Advances in Neural Information Processing Systems}, pages
  4126--4135, 2019.

\bibitem{su2019vl}
Weijie Su, Xizhou Zhu, Yue Cao, Bin Li, Lewei Lu, Furu Wei, and Jifeng Dai.
\newblock Vl-bert: Pre-training of generic visual-linguistic representations.
\newblock {\em arXiv preprint arXiv:1908.08530}, 2019.

\bibitem{sundararajan2017axiomatic}
Mukund Sundararajan, Ankur Taly, and Qiqi Yan.
\newblock Axiomatic attribution for deep networks.
\newblock In {\em Proceedings of the 34th International Conference on Machine
  Learning-Volume 70}, pages 3319--3328. JMLR. org, 2017.

\bibitem{tan2019lxmert}
Hao Tan and Mohit Bansal.
\newblock Lxmert: Learning cross-modality encoder representations from
  transformers.
\newblock {\em arXiv preprint arXiv:1908.07490}, 2019.

\bibitem{vaswani2017attention}
Ashish Vaswani, Noam Shazeer, Niki Parmar, Jakob Uszkoreit, Llion Jones,
  Aidan~N Gomez, {\L}ukasz Kaiser, and Illia Polosukhin.
\newblock Attention is all you need.
\newblock In {\em Advances in neural information processing systems}, pages
  5998--6008, 2017.

\bibitem{voita2019analyzing}
Elena Voita, David Talbot, Fedor Moiseev, Rico Sennrich, and Ivan Titov.
\newblock Analyzing multi-head self-attention: Specialized heads do the heavy
  lifting, the rest can be pruned.
\newblock In {\em Proceedings of the 57th Annual Meeting of the Association for
  Computational Linguistics}, pages 5797--5808, 2019.

\bibitem{xu2015show}
Kelvin Xu, Jimmy Ba, Ryan Kiros, Kyunghyun Cho, Aaron Courville, Ruslan
  Salakhudinov, Rich Zemel, and Yoshua Bengio.
\newblock Show, attend and tell: Neural image caption generation with visual
  attention.
\newblock In {\em International conference on machine learning}, pages
  2048--2057, 2015.

\bibitem{zaidan2008modeling}
Omar Zaidan and Jason Eisner.
\newblock Modeling annotators: A generative approach to learning from annotator
  rationales.
\newblock In {\em Proceedings of the 2008 conference on Empirical methods in
  natural language processing}, pages 31--40, 2008.

\bibitem{zeiler2014visualizing}
Matthew~D Zeiler and Rob Fergus.
\newblock Visualizing and understanding convolutional networks.
\newblock In {\em European conference on computer vision}, pages 818--833.
  Springer, 2014.

\bibitem{zhang2018top}
Jianming Zhang, Sarah~Adel Bargal, Zhe Lin, Jonathan Brandt, Xiaohui Shen, and
  Stan Sclaroff.
\newblock Top-down neural attention by excitation backprop.
\newblock {\em International Journal of Computer Vision}, 126(10):1084--1102,
  2018.

\bibitem{zhou2018interpreting}
Bolei Zhou, David Bau, Aude Oliva, and Antonio Torralba.
\newblock Interpreting deep visual representations via network dissection.
\newblock {\em IEEE transactions on pattern analysis and machine intelligence},
  2018.

\bibitem{zhou2016learning}
Bolei Zhou, Aditya Khosla, Agata Lapedriza, Aude Oliva, and Antonio Torralba.
\newblock Learning deep features for discriminative localization.
\newblock In {\em Proceedings of the IEEE conference on computer vision and
  pattern recognition}, pages 2921--2929, 2016.

\end{thebibliography}
}

\clearpage
\onecolumn
\appendix
\setcounter{equation}{0}
\setcounter{figure}{0}

\section{Details of the various Baselines}

\paragraph{\textbf{GradCAM}}As mentioned in Sec.~4, we consider the last attention layer (closest to the output) - namely $\mathbf{A}^{(1)}$. This results in a feature-map of size $h \times s \times s$. Following the process described in Sec.~3.4, we take only the \texttt{[CLS]} token's row (without the \texttt{[CLS]} token's column), and reshape to the patches grid size $h_p \times w_p$. This results in a feature-map similar to the 2D feature-map used for GradCAM, where the number of channels, in this case, is $h$, and the height and width are $h_p$ and $w_p$. The reason we use the last attention layer is because of the sparse gradients issue described in Sec.~4.

\paragraph{\textbf{raw-attention}} The raw-attention method visualizes the last attention layer (closest to the output) - namely $\mathbf{A}^{(1)}$. It follows the process described in Sec.~3.4 to extract the final output.

\paragraph{\textbf{LRP}} In this method, we propagate relevance up to the input image, following the propagation rules of LRP (not our modified rules and normalizations).

\paragraph{\textbf{partial-LRP}} Following~[42], we visualize an intermediate relevance map, more specifically, we visualize the last attention-map's relevance, namely $R^{(n_1)}$, using LRP propagation rules.

\paragraph{\textbf{rollout}} We follow Eq.~16.

\clearpage
\section{Proofs for Lemmas}
Given two tensors $u$ and $v$, we compute the relevance propagation of binary operators (\ie, operators that process two operands) as follows:
\begin{align}
    R_j^{u^{(n)}} &= \mathcal{G}(u,v,R^{(n-1)}) \notag\\
    R_k^{v^{(n)}} &= \mathcal{G}(v,u,R^{(n-1)}) \label{eq:uvr}
\end{align}
where $R_j^{u^{(n)}}$ and $R_k^{v^{(n)}}$ are the relevances for $u$ and $v$ respectively. 

The following lemma shows that for the case of addition, the conservation rule is preserved, i.e., 
\begin{equation}
\sum_j R_j^{u^{(n)}} + \sum_k R_k^{v^{(n)}} = \sum_i R_i^{(n-1)}
\label{eq:binarycr}.
\end{equation}

However, this is not the case for matrix multiplication.

\begin{lemma}
\label{lemma1}
Given two tensors $u$ and $v$, consider the relevances that are computed according to Eq.~\ref{eq:uvr}. Then, (i) if layer $L^{(n)}$ adds the two tensors, i.e.,  $L^{(n)}(u,v) = u+v$ then the conservation rule of Eq.~\ref{eq:binarycr} is maintained. (ii) if the layer performs matrix multiplication $L^{(n)}(u,v) = uv$, then Eq.~\ref{eq:binarycr} does not hold in general.
\end{lemma}
\begin{proof}
For part (i), we note that the number of elements in $u$ equals the number of elements in $v$, therefore $k=j$, and we can write Eq.~\ref{eq:binarycr} following the definition of $\mathcal{G}$:
\begin{align}
    \nonumber
    &\sum_j\sum_i u_j\frac{\partial (u_i+v_i)}{\partial u_j} \frac{R^{(n-1)}_i}{u_{i}+v_{i}} + \sum_j\sum_i v_j\frac{\partial (u_i+v_i)}{\partial v_j} \frac{R^{(n-1)}_i}{u_{i}+v_{i}}\\
    \nonumber
    &= \sum_j \frac{u_j}{u_j+v_j} R^{(n-1)}_j + \sum_j \frac{v_j}{u_j+v_j} R^{(n-1)}_j\\
    &= \sum_j \frac{u_j + v_j}{u_j+v_j} R^{(n-1)}_j = \sum_j R_j^{(n-1)}
\end{align}
note that, in this case, it is possible that $\sum_j R_j^{u^{(n)}}\ne\sum_j R_j^{v^{(n)}}$.

{As shown in the main text, while the sum of two tensors maintains the conservation rule, their values may explode. Consider $u = \m{e^a\\e^b}$, $v = \m{1-e^a\\1-e^b}$ and $R = \m{1\\1}$, following the definition of $\mathcal{G}$ we have:
\begin{align}
    \nonumber
    R_j^{u^{(n)}} &=\sum_i u_j\frac{\partial (u_i+v_i)}{\partial u_j} \frac{R^{(n-1)}_i}{u_{i}+v_{i}} = \frac{u_j}{u_j+v_j} R^{(n-1)}_j, \quad  R_j^{v^{(n)}} = \frac{v_j}{u_j+v_j} R^{(n-1)}_j\\
    R^u &= \m{\frac{e^a}{e^a - e^a + 1}1\\\frac{e^b}{e^b - e^b + 1}1} = \m{e^a\\e^b}, \quad  R^v = \m{1-e^a\\1-e^b}
\end{align}
which causes numerical instability.}

For part (ii), in the case of matrix multiplication between $u$ and $v$, where $u \in \mathbb{R}^{k,m},\, v \in \mathbb{R}^{m,l}$,  we will show that: $\sum_k \sum_m R^{u^{(n)}}_{k, m} = \sum_m \sum_l R^{v^{(n)}}_{m, l}= \sum_l \sum_k R^{(n-1)}_{k,l}$, which invalidates the conservation rule:
\begin{align}
    R^{u^{(n)}}_{k, m} &= \sum_l u_{k,m} \frac{\partial(uv)_{k,l}}{\partial u_{k,m}} \frac{R^{(n-1)}_{k,l}}{\sum_{m^\prime}u_{k,m^\prime}v_{m^\prime l}} 
    = \sum_l \frac{u_{k,m}v_{m,l}}{\sum_{m^\prime}u_{k,m^\prime}v_{m^\prime l}} R^{(n-1)}_{k,l}\\
    R^{v^{(n)}}_{m, l} &= \sum_k v_{m,l} \frac{\partial(uv)_{k,l}}{\partial v_{m,l}} \frac{R^{(n-1)}_{k,l}}{\sum_{m^\prime}u_{k,m^\prime}v_{m^\prime l}}
    = \sum_k \frac{u_{k,m}v_{m,l}}{\sum_{m^\prime}u_{k,m^\prime}v_{m^\prime l}} R^{(n-1)}_{k,l}\\
    \nonumber
    \sum_k \sum_m R^{u^{(n)}}_{k, m} + \sum_m \sum_l R^{v^{(n)}}_{m, l} &= \sum_k \sum_m \sum_l \frac{u_{k,m}v_{m,l}}{\sum_{m^\prime}u_{k,m^\prime}v_{m^\prime l}} R^{(n-1)}_{k,l} + \sum_m \sum_l \sum_k \frac{u_{k,m}v_{m,l}}{\sum_{m^\prime}u_{k,m^\prime}v_{m^\prime l}} R^{(n-1)}_{k,l}\\
    &= \sum_k \sum_l \frac{\sum_m u_{k,m}v_{m,l}}{\sum_{m^\prime}u_{k,m^\prime}v_{m^\prime l}} R^{(n-1)}_{k,l} + \sum_l \sum_k \frac{\sum_m u_{k,m}v_{m,l}}{\sum_{m^\prime}u_{k,m^\prime}v_{m^\prime l}} R^{(n-1)}_{k,l} \\
    &= 2 \sum_l \sum_k R^{(n-1)}_{k,l}
\end{align}
\end{proof}

To address the lack of conservation in the attention mechanism, which employs multiplication, and the numerical issues of the skip connections, our method applies a normalization to $R_j^{u^{(n)}}$ and $R_k^{v^{(n)}}$:
\begin{align}
    \nonumber
    \bar{R}_j^{u^{(n)}} &= R_j^{u^{(n)}} \frac{\abs{\sum_j {R_j^{u^{(n)}}}}}{\abs{\sum_j {R_j^{u^{(n)}}}} + \abs{\sum_k{R_k^{v^{(n)}}}}} \cdot \frac{\sum_i R_i^{(n-1)}}{\sum_j R_j^{u^{(n)}}}\\
    \bar{R}_k^{v^{(n)}} &= R_k^{v^{(n)}} \frac{\abs{\sum_k {R_k^{v^{(n)}}}}}{\abs{\sum_j {R_j^{u^{(n)}}}} + \abs{\sum_k {R_k^{v^{(n)}}}}} \cdot \frac{\sum_i R_i^{(n-1)}}{\sum_k R_k^{v^{(n)}}}
    \label{eq:norm}
\end{align}

\begin{lemma}
\label{lemma2}
The normalization technique upholds the following properties: (i) it maintains the conservation rule, i.e.: $\sum_j \bar{R}_j^{u^{(n)}} + \sum_k \bar{R}_k^{v^{(n)}} =  \sum_i R_i^{(n-1)}$, (ii) it bounds the relevance sum of each tensor such that:
\begin{align}
    0 \le \sum_j \bar{R}_j^{u^{(n)}}, \sum_k \bar{R}_k^{v^{(n)}} \le  \sum_i R_i^{(n-1)}
\end{align}
\end{lemma}
\begin{proof}
For part (i), it holds that:
\begin{align}
    &\sum_j \bar{R}_j^{u^{(n)}} + \sum_k \bar{R}_k^{v^{(n)}}\\
    \nonumber
    &=\sum_j R_j^{u^{(n)}} \frac{\abs{\sum_j {R_j^{u^{(n)}}}}}{\abs{\sum_j {R_j^{u^{(n)}}}} + \abs{\sum_k{R_k^{v^{(n)}}}}} \cdot \frac{\sum_i R_i^{(n-1)}}{\sum_j R_j^{u^{(n)}}} \\
    &+ \sum_k R_k^{v^{(n)}} \frac{\abs{\sum_k {R_k^{v^{(n)}}}}}{\abs{\sum_j {R_j^{u^{(n)}}}} + \abs{\sum_k {R_k^{v^{(n)}}}}} \cdot \frac{\sum_i R_i^{(n-1)}}{\sum_k R_k^{v^{(n)}}}\\
    &= \frac{\abs{\sum_j {R_j^{u^{(n)}}}}}{\abs{\sum_j {R_j^{u^{(n)}}}} + \abs{\sum_k{R_k^{v^{(n)}}}}} \cdot \sum_i R_i^{(n-1)} + \frac{\abs{\sum_k {R_k^{v^{(n)}}}}}{\abs{\sum_j {R_j^{u^{(n)}}}} + \abs{\sum_k {R_k^{v^{(n)}}}}} \cdot \sum_i R_i^{(n-1)}\\
    &= \frac{\abs{\sum_j {R_j^{u^{(n)}}}} + \abs{\sum_k {R_k^{v^{(n)}}}}}{\abs{\sum_j {R_j^{u^{(n)}}}} + \abs{\sum_k{R_k^{v^{(n)}}}}} \cdot \sum_i R_i^{(n-1)} = \sum_i R_i^{(n-1)}
\end{align}

For part (ii) it is trivial to see that we weigh each tensor according to its relative absolute-value contribution:
\begin{align}
    \sum_j \bar{R}_j^{u^{(n)}} &= \sum_j R_j^{u^{(n)}} \frac{\abs{\sum_j {R_j^{u^{(n)}}}}}{\abs{\sum_j {R_j^{u^{(n)}}}} + \abs{\sum_k{R_k^{v^{(n)}}}}} \cdot \frac{\sum_i R_i^{(n-1)}}{\sum_j R_j^{u^{(n)}}}\\
    & = \frac{\abs{\sum_j {R_j^{u^{(n)}}}}}{\abs{\sum_j {R_j^{u^{(n)}}}} + \abs{\sum_k{R_k^{v^{(n)}}}}} \cdot \sum_i R_i^{(n-1)}
\end{align}
we see that:
\begin{align}
    0 \le \frac{\abs{\sum_j {R_j^{u^{(n)}}}}}{\abs{\sum_j {R_j^{u^{(n)}}}} + \abs{\sum_k{R_k^{v^{(n)}}}}} \le 1
\end{align}
therefore:
\begin{align}
    0 \le \sum_j \bar{R}_j^{u^{(n)}}, \sum_k \bar{R}_k^{v^{(n)}} \le  \sum_i R_i^{(n-1)}
\end{align}
\end{proof}

\clearpage
\section{Visualizations - Multiple-class Images}
\begin{figure}[h]
    \centering
    \begin{tabular*}{\linewidth}{@{\extracolsep{\fill}}ccccccc}
    Input &rollout & raw-attention & GradCAM & LRP & partial LRP & Ours\\
    \raisebox{9mm}{\multirow{2}{*}{\makecell{Dog $\rightarrow$\\\includegraphics[width=0.13\textwidth]{figures/class/catdog.png}\\ Cat $\rightarrow$}}} &
    \includegraphics[width=0.12\linewidth]{figures/class/dog_rollout.png} &
    \includegraphics[width=0.12\linewidth]{figures/class/dog_last_layer_attn.png} &
    \includegraphics[width=0.12\linewidth]{figures/class/dog_attn_gradcam.png} &
    \includegraphics[width=0.12\linewidth]{figures/class/dog_full_LRP.png} &
    \includegraphics[width=0.12\linewidth]{figures/class/dog_LRP_last_layer.png} &
    \includegraphics[width=0.12\linewidth]{figures/class/dog_grad_LRP.png}\\
    &
    \includegraphics[width=0.12\linewidth]{figures/class/cat_rollout.png} &
    \includegraphics[width=0.12\linewidth]{figures/class/cat_last_layer_attn.png} &
    \includegraphics[width=0.12\linewidth]{figures/class/cat_attn_gradcam.png} &
    \includegraphics[width=0.12\linewidth]{figures/class/cat_full_LRP.png} &
    \includegraphics[width=0.12\linewidth]{figures/class/cat_LRP_last_layer.png} &
    \includegraphics[width=0.12\linewidth]{figures/class/cat_grad_LRP.png}\\
    
    \raisebox{9mm}{\multirow{2}{*}{\makecell{Elephant $\rightarrow$\\\includegraphics[width=0.13\textwidth]{figures/class/el2.png}\\ Zebra $\rightarrow$}}} &
    \includegraphics[width=0.12\linewidth]{figures/class/elephant_rollout.png} &
    \includegraphics[width=0.12\linewidth]{figures/class/elephant_last_layer_attn.png} &
    \includegraphics[width=0.12\linewidth]{figures/class/elephant_attn_gradcam.png} &
    \includegraphics[width=0.12\linewidth]{figures/class/elephant_full_LRP.png} &
    \includegraphics[width=0.12\linewidth]{figures/class/elephant_LRP_last_layer.png} &
    \includegraphics[width=0.12\linewidth]{figures/class/elephant_grad_LRP.png}\\
    &
    \includegraphics[width=0.12\linewidth]{figures/class/zebra_rollout.png} &
    \includegraphics[width=0.12\linewidth]{figures/class/zebra_last_layer_attn.png} &
    \includegraphics[width=0.12\linewidth]{figures/class/zebra_attn_gradcam.png} &
    \includegraphics[width=0.12\linewidth]{figures/class/zebra_full_LRP.png} &
    \includegraphics[width=0.12\linewidth]{figures/class/zebra_LRP_last_layer.png} &
    \includegraphics[width=0.12\linewidth]{figures/class/zebra_grad_LRP.png}\\
    
    \raisebox{9mm}{\multirow{2}{*}{\makecell{Elephant $\rightarrow$\\\includegraphics[width=0.13\textwidth]{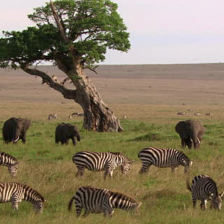}\\ Zebra $\rightarrow$}}} &
    \includegraphics[width=0.12\linewidth]{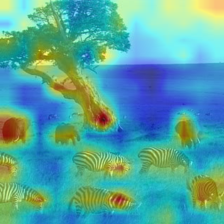} &
    \includegraphics[width=0.12\linewidth]{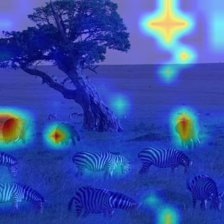} &
    \includegraphics[width=0.12\linewidth]{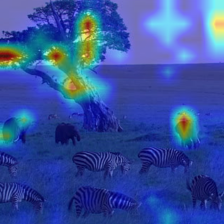} &
    \includegraphics[width=0.12\linewidth]{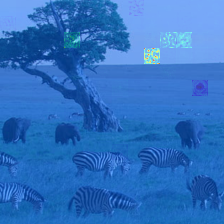} &
    \includegraphics[width=0.12\linewidth]{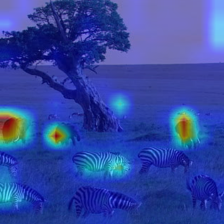} &
    \includegraphics[width=0.12\linewidth]{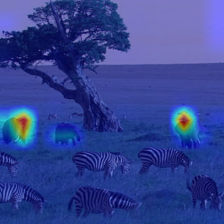}\\
    &
    \includegraphics[width=0.12\linewidth]{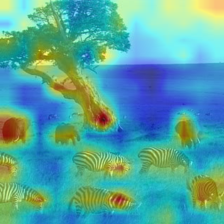} &
    \includegraphics[width=0.12\linewidth]{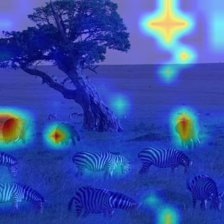} &
    \includegraphics[width=0.12\linewidth]{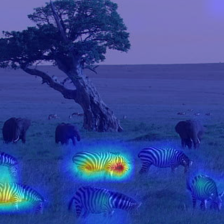} &
    \includegraphics[width=0.12\linewidth]{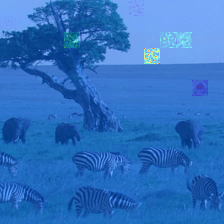} &
    \includegraphics[width=0.12\linewidth]{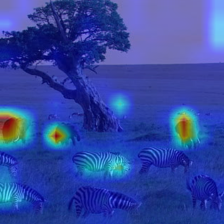} &
    \includegraphics[width=0.12\linewidth]{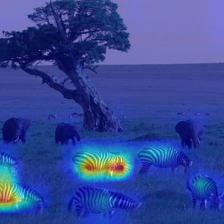}\\
    
    \raisebox{9mm}{\multirow{2}{*}{\makecell{Elephant $\rightarrow$\\\includegraphics[width=0.13\textwidth]{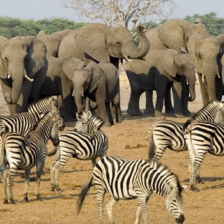}\\ Zebra $\rightarrow$}}} &
    \includegraphics[width=0.12\linewidth]{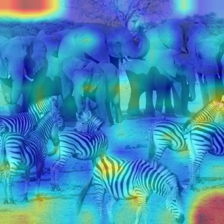} &
    \includegraphics[width=0.12\linewidth]{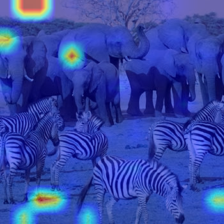} &
    \includegraphics[width=0.12\linewidth]{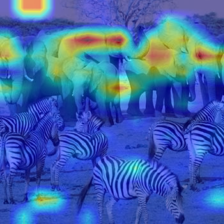} &
    \includegraphics[width=0.12\linewidth]{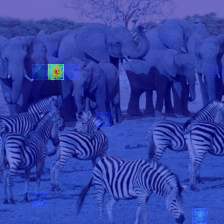} &
    \includegraphics[width=0.12\linewidth]{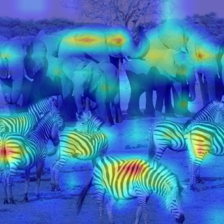} &
    \includegraphics[width=0.12\linewidth]{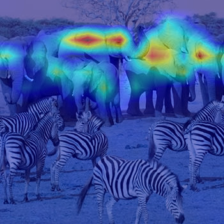}\\
    &
    \includegraphics[width=0.12\linewidth]{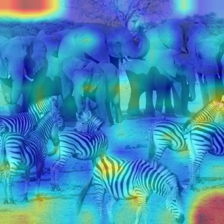} &
    \includegraphics[width=0.12\linewidth]{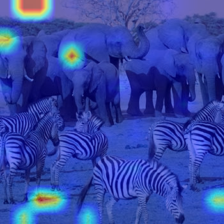} &
    \includegraphics[width=0.12\linewidth]{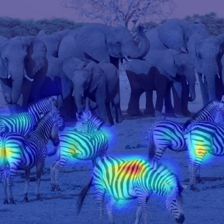} &
    \includegraphics[width=0.12\linewidth]{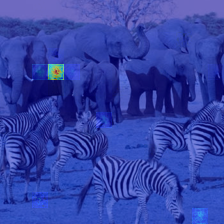} &
    \includegraphics[width=0.12\linewidth]{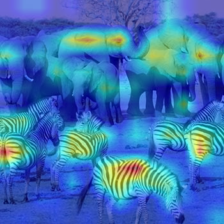} &
    \includegraphics[width=0.12\linewidth]{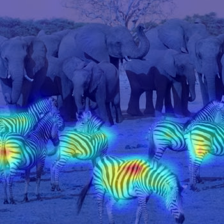}
    \end{tabular*}
    \caption{Multiple-class visualization. For each input image, we visualize two different classes. As can be seen, only our method and GradCAM produce class-specific visualisations, where our method has fewer artifacts, and captures the objects more completely.}
    \label{fig:multi1}
\end{figure}

\clearpage
\begin{figure}[h]
    \centering
    \begin{tabular*}{\linewidth}{@{\extracolsep{\fill}}ccccccc}
    Input &rollout & raw-attention & GradCAM & LRP & partial LRP & Ours\\
    \raisebox{9mm}{\multirow{2}{*}{\makecell{Elephant $\rightarrow$\\\includegraphics[width=0.13\textwidth]{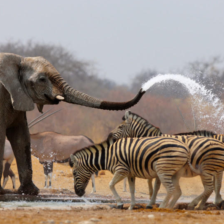}\\ Zebra $\rightarrow$}}} &
    \includegraphics[width=0.12\linewidth]{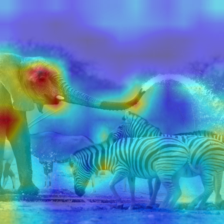} &
    \includegraphics[width=0.12\linewidth]{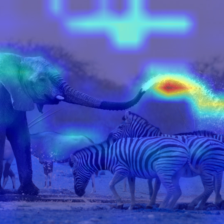} &
    \includegraphics[width=0.12\linewidth]{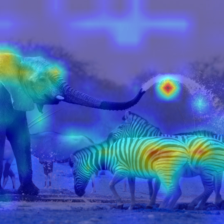} &
    \includegraphics[width=0.12\linewidth]{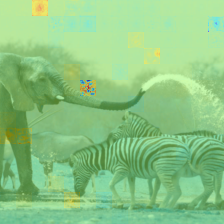} &
    \includegraphics[width=0.12\linewidth]{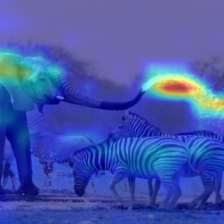} &
    \includegraphics[width=0.12\linewidth]{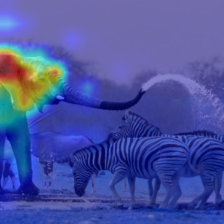}\\
    &
    \includegraphics[width=0.12\linewidth]{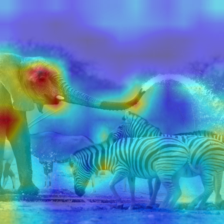} &
    \includegraphics[width=0.12\linewidth]{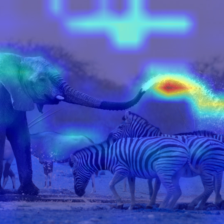} &
    \includegraphics[width=0.12\linewidth]{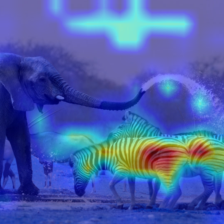} &
    \includegraphics[width=0.12\linewidth]{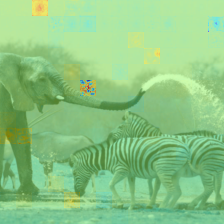} &
    \includegraphics[width=0.12\linewidth]{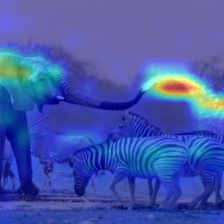} &
    \includegraphics[width=0.12\linewidth]{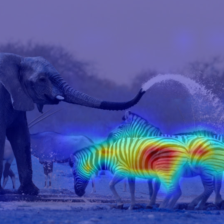}\\
    
    \raisebox{9mm}{\multirow{2}{*}{\makecell{Elephant $\rightarrow$\\\includegraphics[width=0.13\textwidth]{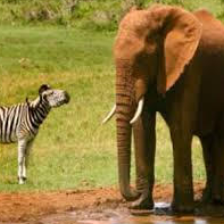}\\ Zebra $\rightarrow$}}} &
    \includegraphics[width=0.12\linewidth]{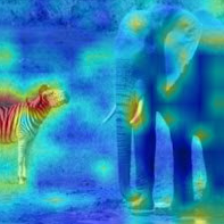} &
    \includegraphics[width=0.12\linewidth]{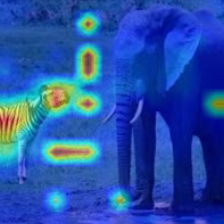} &
    \includegraphics[width=0.12\linewidth]{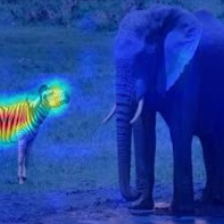} &
    \includegraphics[width=0.12\linewidth]{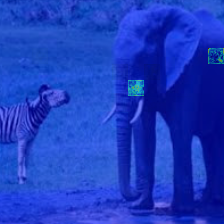} &
    \includegraphics[width=0.12\linewidth]{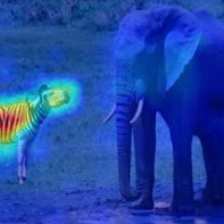} &
    \includegraphics[width=0.12\linewidth]{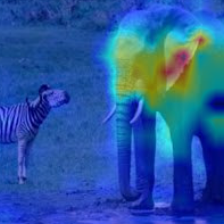}\\
    &
    \includegraphics[width=0.12\linewidth]{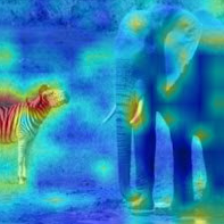} &
    \includegraphics[width=0.12\linewidth]{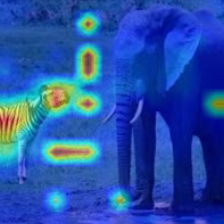} &
    \includegraphics[width=0.12\linewidth]{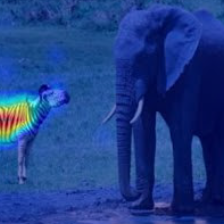} &
    \includegraphics[width=0.12\linewidth]{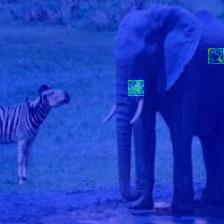} &
    \includegraphics[width=0.12\linewidth]{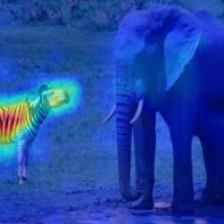} &
    \includegraphics[width=0.12\linewidth]{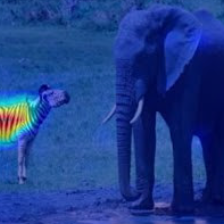}\\
    
    \raisebox{9mm}{\multirow{2}{*}{\makecell{Dog $\rightarrow$\\\includegraphics[width=0.13\textwidth]{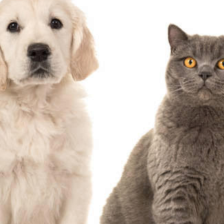}\\ Cat $\rightarrow$}}} &
    \includegraphics[width=0.12\linewidth]{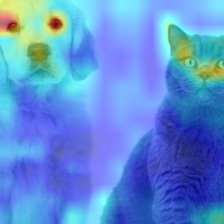} &
    \includegraphics[width=0.12\linewidth]{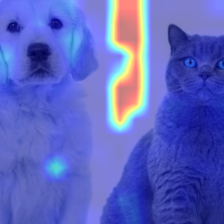} &
    \includegraphics[width=0.12\linewidth]{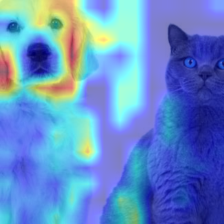} &
    \includegraphics[width=0.12\linewidth]{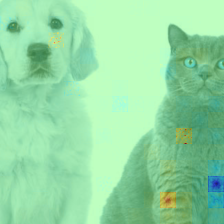} &
    \includegraphics[width=0.12\linewidth]{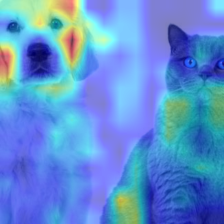} &
    \includegraphics[width=0.12\linewidth]{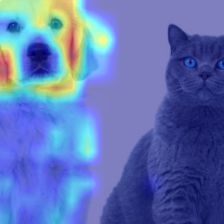}\\
    &
    \includegraphics[width=0.12\linewidth]{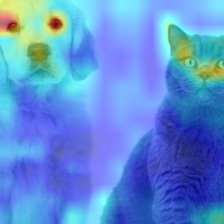} &
    \includegraphics[width=0.12\linewidth]{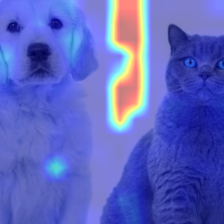} &
    \includegraphics[width=0.12\linewidth]{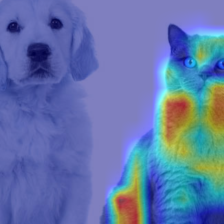} &
    \includegraphics[width=0.12\linewidth]{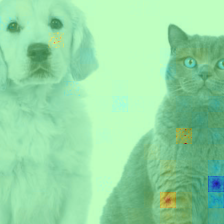} &
    \includegraphics[width=0.12\linewidth]{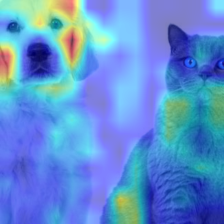} &
    \includegraphics[width=0.12\linewidth]{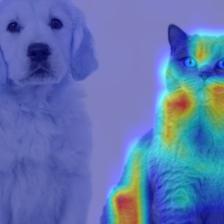}
    \end{tabular*}
    \caption{Multiple-class visualization. For each input image, we visualize two different classes. As can be seen, only our method and GradCAM produce class-specific visualisations, where our method has fewer artifacts, and captures the objects more completely.}
    \label{fig:multi2}
\end{figure}

\clearpage
\section{Visualizations - Single-class Images}
\begin{figure}[h]
\centering
\begin{tabular*}{\linewidth}{l}
~~~~~~~~~~~~~~~~Input~~~~~~~~~~~~~~~~rollout~~~~~~~~~raw-attention~~~~~~~GradCAM~~~~~~~~~~~~~LRP~~~~~~~~~~~~~partial LRP~~~~~~~~~~~~Ours\\
\end{tabular*}
\includegraphics[width=.9\linewidth]{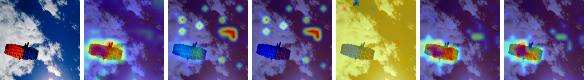}\\
\includegraphics[width=.9\linewidth]{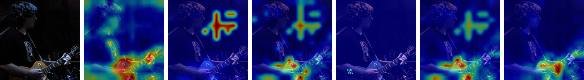}\\
\includegraphics[width=.9\linewidth]{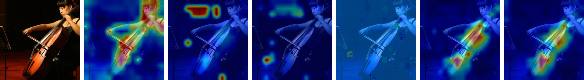}\\
\includegraphics[width=.9\linewidth]{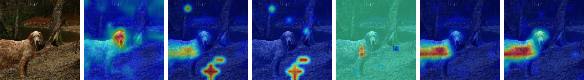}\\
\includegraphics[width=.9\linewidth]{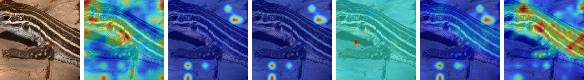}\\
\includegraphics[width=.9\linewidth]{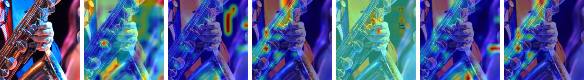}\\
\includegraphics[width=.9\linewidth]{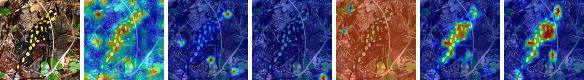}\\
\includegraphics[width=.9\linewidth]{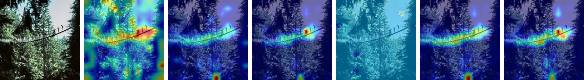}\\
\includegraphics[width=.9\linewidth]{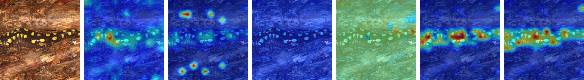}\\
\caption{Sample images from ImageNet val-set.}
\label{fig:single_class_0}
\end{figure}

\clearpage
\begin{figure}[h]
\centering
\begin{tabular*}{\linewidth}{l}
~~~~~~~~~~~~~~~~Input~~~~~~~~~~~~~~~~rollout~~~~~~~~~raw-attention~~~~~~~GradCAM~~~~~~~~~~~~~LRP~~~~~~~~~~~~~partial LRP~~~~~~~~~~~~Ours\\
\end{tabular*}
\includegraphics[width=.9\linewidth]{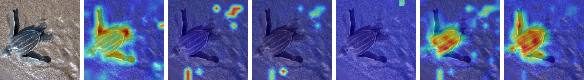}\\
\includegraphics[width=.9\linewidth]{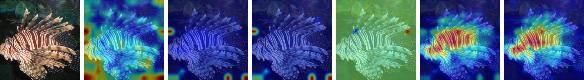}\\
\includegraphics[width=.9\linewidth]{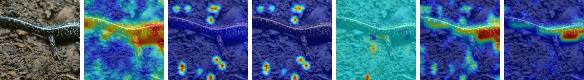}\\
\includegraphics[width=.9\linewidth]{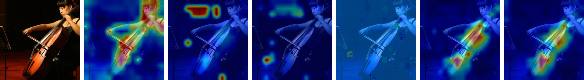}\\
\includegraphics[width=.9\linewidth]{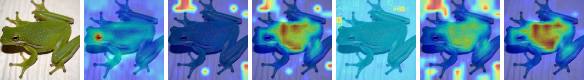}\\
\includegraphics[width=.9\linewidth]{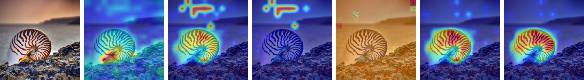}\\
\includegraphics[width=.9\linewidth]{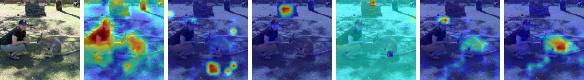}\\
\includegraphics[width=.9\linewidth]{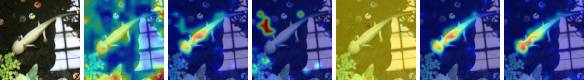}\\
\includegraphics[width=.9\linewidth]{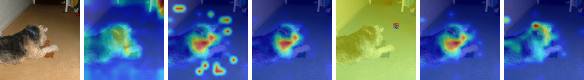}\\
\caption{Sample images from ImageNet val-set.}
\label{fig:single_class_1}
\end{figure}

\clearpage
\begin{figure}[h]
\centering
\begin{tabular*}{\linewidth}{l}
~~~~~~~~~~~~~~~~Input~~~~~~~~~~~~~~~~rollout~~~~~~~~~raw-attention~~~~~~~GradCAM~~~~~~~~~~~~~LRP~~~~~~~~~~~~~partial LRP~~~~~~~~~~~~Ours\\
\end{tabular*}
\includegraphics[width=.9\linewidth]{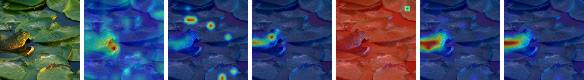}\\
\includegraphics[width=.9\linewidth]{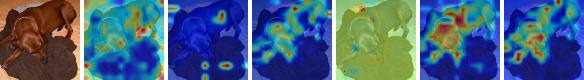}\\
\includegraphics[width=.9\linewidth]{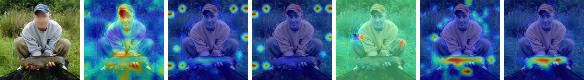}\\
\includegraphics[width=.9\linewidth]{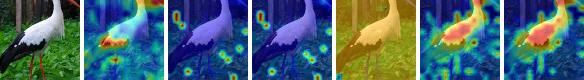}\\
\includegraphics[width=.9\linewidth]{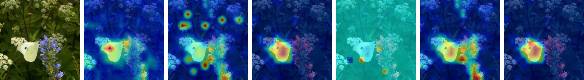}\\
\includegraphics[width=.9\linewidth]{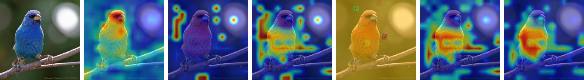}\\
\includegraphics[width=.9\linewidth]{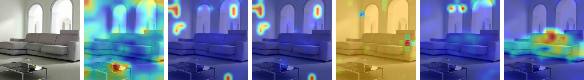}\\
\includegraphics[width=.9\linewidth]{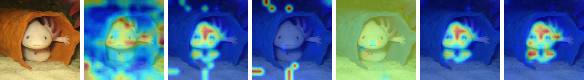}\\
\includegraphics[width=.9\linewidth]{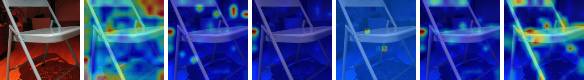}\\
\caption{Sample images from ImageNet val-set.}
\label{fig:single_class_2}
\end{figure}

\clearpage
\begin{figure}[h]
\centering
\begin{tabular*}{\linewidth}{l}
~~~~~~~~~~~~~~~~Input~~~~~~~~~~~~~~~~rollout~~~~~~~~~raw-attention~~~~~~~GradCAM~~~~~~~~~~~~~LRP~~~~~~~~~~~~~partial LRP~~~~~~~~~~~~Ours\\
\end{tabular*}
\includegraphics[width=.9\linewidth]{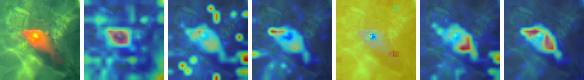}\\
\includegraphics[width=.9\linewidth]{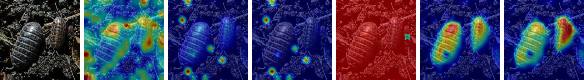}\\
\includegraphics[width=.9\linewidth]{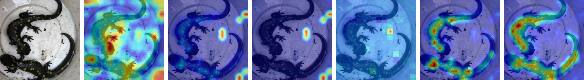}\\
\includegraphics[width=.9\linewidth]{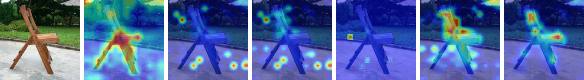}\\
\includegraphics[width=.9\linewidth]{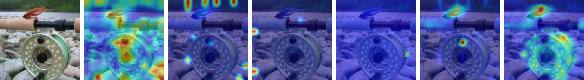}\\
\includegraphics[width=.9\linewidth]{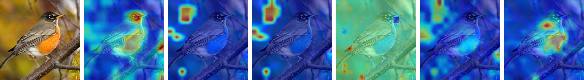}\\
\includegraphics[width=.9\linewidth]{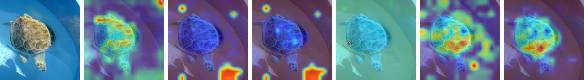}\\
\includegraphics[width=.9\linewidth]{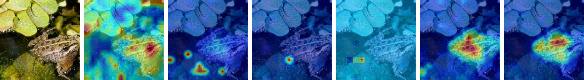}\\
\includegraphics[width=.9\linewidth]{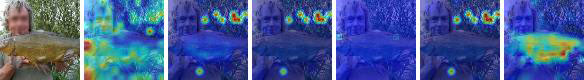}\\
\caption{Sample images from ImageNet val-set.}
\label{fig:single_class_3}
\end{figure}

\clearpage
\begin{figure}[h]
\centering
\begin{tabular*}{\linewidth}{l}
~~~~~~~~~~~~~~~~Input~~~~~~~~~~~~~~~~rollout~~~~~~~~~raw-attention~~~~~~~GradCAM~~~~~~~~~~~~~LRP~~~~~~~~~~~~~partial LRP~~~~~~~~~~~~Ours\\
\end{tabular*}
\includegraphics[width=.9\linewidth]{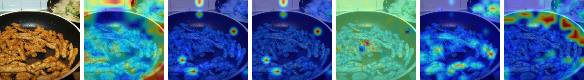}\\
\includegraphics[width=.9\linewidth]{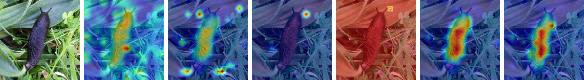}\\
\includegraphics[width=.9\linewidth]{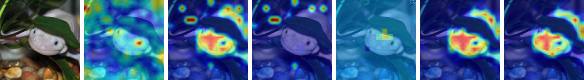}\\
\includegraphics[width=.9\linewidth]{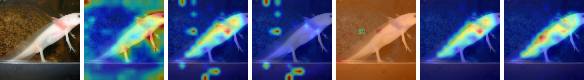}\\
\includegraphics[width=.9\linewidth]{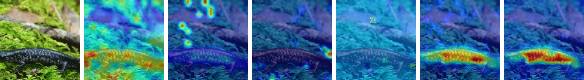}\\
\includegraphics[width=.9\linewidth]{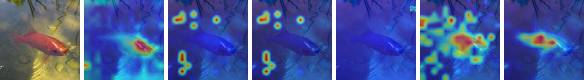}\\
\includegraphics[width=.9\linewidth]{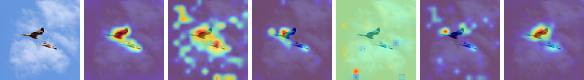}\\
\includegraphics[width=.9\linewidth]{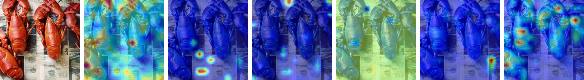}\\
\includegraphics[width=.9\linewidth]{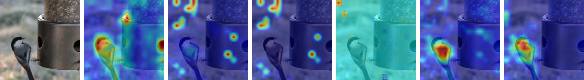}\\
\caption{Sample images from ImageNet val-set.}
\label{fig:single_class_4}
\end{figure}

\clearpage
\begin{figure}[h]
\centering
\begin{tabular*}{\linewidth}{l}
~~~~~~~~~~~~~~~~Input~~~~~~~~~~~~~~~~rollout~~~~~~~~~raw-attention~~~~~~~GradCAM~~~~~~~~~~~~~LRP~~~~~~~~~~~~~partial LRP~~~~~~~~~~~~Ours\\
\end{tabular*}
\includegraphics[width=.9\linewidth]{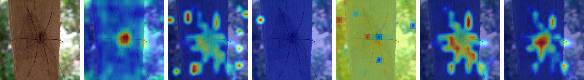}\\
\includegraphics[width=.9\linewidth]{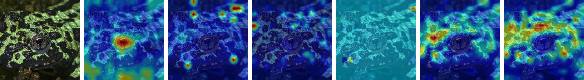}\\
\includegraphics[width=.9\linewidth]{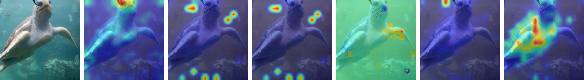}\\
\includegraphics[width=.9\linewidth]{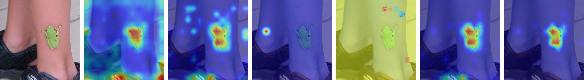}\\
\includegraphics[width=.9\linewidth]{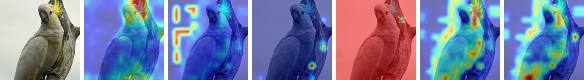}\\
\includegraphics[width=.9\linewidth]{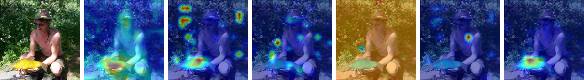}\\
\includegraphics[width=.9\linewidth]{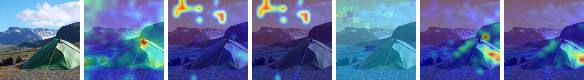}\\
\includegraphics[width=.9\linewidth]{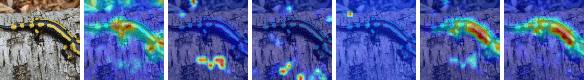}\\
\includegraphics[width=.9\linewidth]{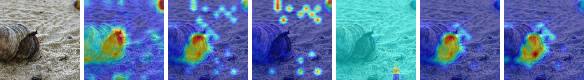}\\
\caption{Sample images from ImageNet val-set.}
\label{fig:single_class_5}
\end{figure}

\clearpage
\begin{figure}[h]
\centering
\begin{tabular*}{\linewidth}{l}
~~~~~~~~~~~~~~~~Input~~~~~~~~~~~~~~~~rollout~~~~~~~~~raw-attention~~~~~~~GradCAM~~~~~~~~~~~~~LRP~~~~~~~~~~~~~partial LRP~~~~~~~~~~~~Ours\\
\end{tabular*}
\includegraphics[width=.9\linewidth]{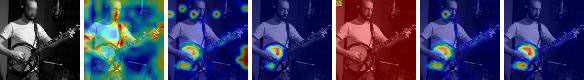}\\
\includegraphics[width=.9\linewidth]{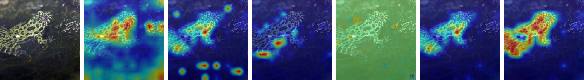}\\
\includegraphics[width=.9\linewidth]{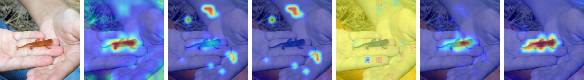}\\
\includegraphics[width=.9\linewidth]{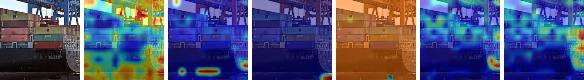}\\
\includegraphics[width=.9\linewidth]{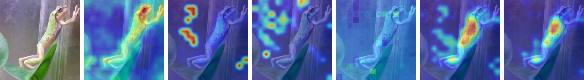}\\
\includegraphics[width=.9\linewidth]{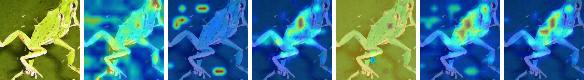}\\
\includegraphics[width=.9\linewidth]{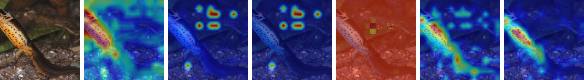}\\
\includegraphics[width=.9\linewidth]{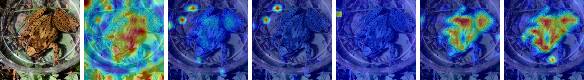}\\
\includegraphics[width=.9\linewidth]{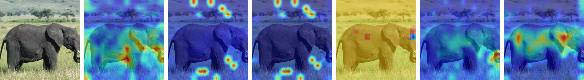}\\
\caption{Sample images from ImageNet val-set.}
\label{fig:single_class_6}
\end{figure}

\clearpage
\begin{figure}[h]
\centering
\begin{tabular*}{\linewidth}{l}
~~~~~~~~~~~~~~~~Input~~~~~~~~~~~~~~~~rollout~~~~~~~~~raw-attention~~~~~~~GradCAM~~~~~~~~~~~~~LRP~~~~~~~~~~~~~partial LRP~~~~~~~~~~~~Ours\\
\end{tabular*}
\includegraphics[width=.9\linewidth]{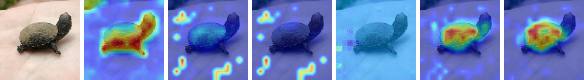}\\
\includegraphics[width=.9\linewidth]{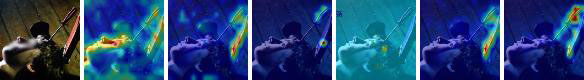}\\
\includegraphics[width=.9\linewidth]{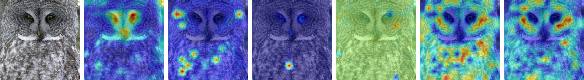}\\
\includegraphics[width=.9\linewidth]{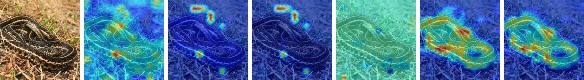}\\
\includegraphics[width=.9\linewidth]{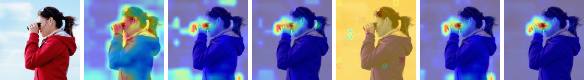}\\
\includegraphics[width=.9\linewidth]{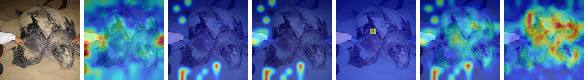}\\
\includegraphics[width=.9\linewidth]{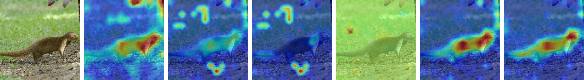}\\
\includegraphics[width=.9\linewidth]{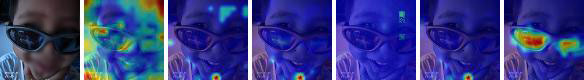}\\
\includegraphics[width=.9\linewidth]{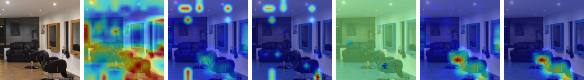}\\
\caption{Sample images from ImageNet val-set.}
\label{fig:single_class_7}
\end{figure}

\clearpage
\begin{figure}[h]
\centering
\begin{tabular*}{\linewidth}{l}
~~~~~~~~~~~~~~~~Input~~~~~~~~~~~~~~~~rollout~~~~~~~~~raw-attention~~~~~~~GradCAM~~~~~~~~~~~~~LRP~~~~~~~~~~~~~partial LRP~~~~~~~~~~~~Ours\\
\end{tabular*}
\includegraphics[width=.9\linewidth]{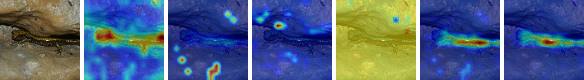}\\
\includegraphics[width=.9\linewidth]{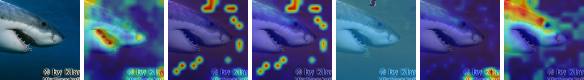}\\
\includegraphics[width=.9\linewidth]{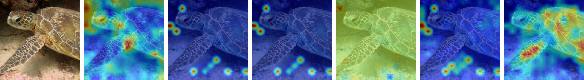}\\
\includegraphics[width=.9\linewidth]{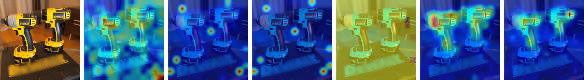}\\
\includegraphics[width=.9\linewidth]{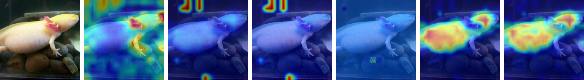}\\
\includegraphics[width=.9\linewidth]{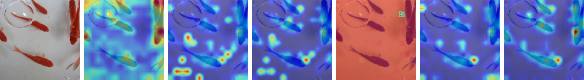}\\
\includegraphics[width=.9\linewidth]{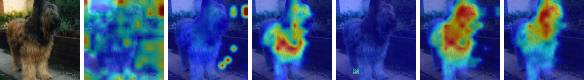}\\
\includegraphics[width=.9\linewidth]{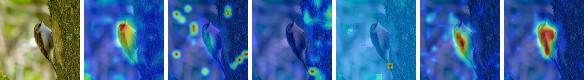}\\
\includegraphics[width=.9\linewidth]{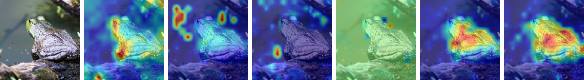}\\
\caption{Sample images from ImageNet val-set.}
\label{fig:single_class_8}
\end{figure}

\clearpage
\begin{figure}[h]
\centering
\begin{tabular*}{\linewidth}{l}
~~~~~~~~~~~~~~~~Input~~~~~~~~~~~~~~~~rollout~~~~~~~~~raw-attention~~~~~~~GradCAM~~~~~~~~~~~~~LRP~~~~~~~~~~~~~partial LRP~~~~~~~~~~~~Ours\\
\end{tabular*}
\includegraphics[width=.9\linewidth]{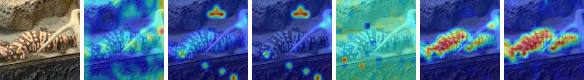}\\
\includegraphics[width=.9\linewidth]{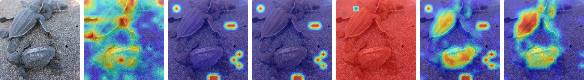}\\
\includegraphics[width=.9\linewidth]{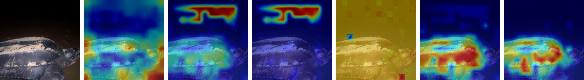}\\
\includegraphics[width=.9\linewidth]{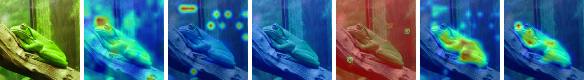}\\
\includegraphics[width=.9\linewidth]{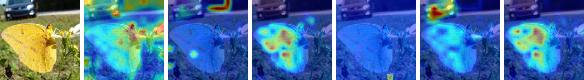}\\
\includegraphics[width=.9\linewidth]{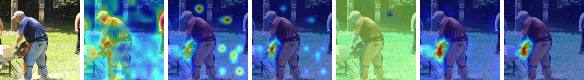}\\
\includegraphics[width=.9\linewidth]{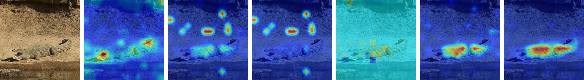}\\
\includegraphics[width=.9\linewidth]{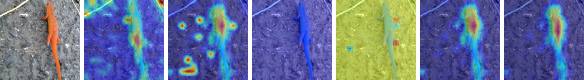}\\
\includegraphics[width=.9\linewidth]{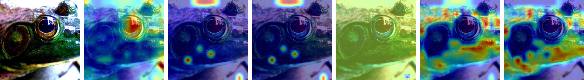}\\
\caption{Sample images from ImageNet val-set.}
\label{fig:single_class_9}
\end{figure}

\clearpage
\section{Visualizations - Text}

In the following visualizations, we use the TAHV heatmap generator for text (\url{https://github.com/jiesutd/Text-Attention-Heatmap-Visualization}) to present the relevancy scores for each method, as well as the excerpts marked by humans. For methods that are class-dependent, we present the attributions obtained both for the ground truth class and the counter-factual class.

Evidently, our method is the only one that is able to present support for both sides, see panels (b,c) of each image. GradCAM often suffers from highlighting the evidence in the opposite direction (sign reversal), e.g., Fig.~\ref{fig:my_label1}(g), in which the counter-factual explanation of GradCAM supports the negative, ground truth, sentiment and not the positive one. 

Partial LRP (panels d,e) is not class-specific in practice. This provides it with an advantage in the quantitative experiments: Partial LRP highlights words with both positive and negative connotations from the same sentence, which better matches the behavior of the human annotators who are asked to mark complete sentences. 

Notice that in most visualizations, it seems that the rollout method focuses mostly on the separation token [SEP], and fails to generate meaningful visualizations. This corresponds to the results presented in the quantitative experiments.
 
It seems from our results, e.g., Fig.~\ref{fig:my_label1}(b,c) that the BERT tokenizer leads to unintuitive results. For example, ``joyless'' is broken down into ``joy'' and ``less'', each supporting different sides of the decision.

~\\
~\\
~\\
~\\
~\\
~\\
~\\
~\\
~\\
~\\
~\\
~\\
~\\
~\\
~\\
~\\
~\\
~\\
~\\
~\\
~\\
~\\
~\\

\captionsetup[figure]{font=scriptsize,labelfont=scriptsize}

\begin{figure}[h]
    \centering
    \includegraphics[width=0.99\linewidth]{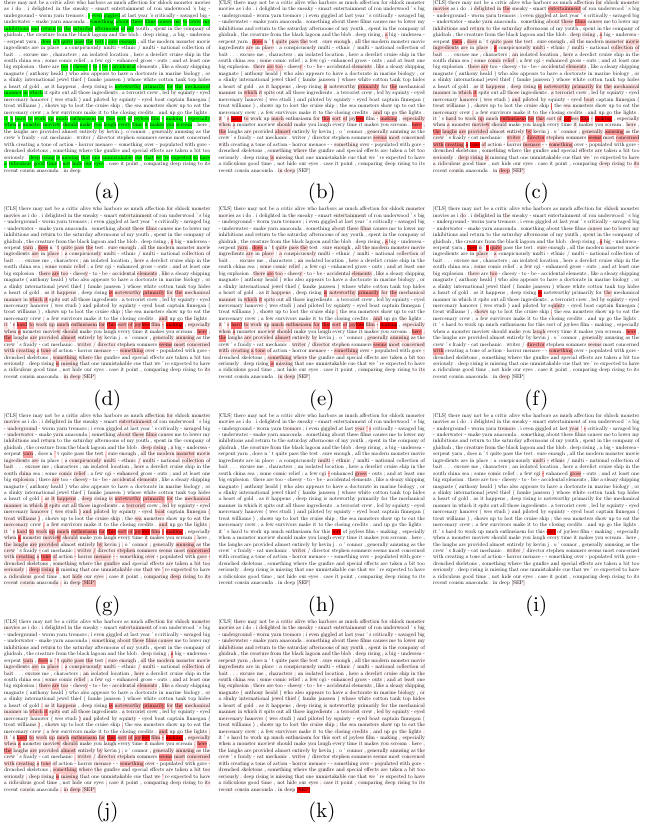}
    \caption{A visualization of the results. For methods that rely on a specific class for propagation, we present both the ground-truth and counter-factual results. The rollout method, as well as the raw attention methods, are class-agnostic. Some words are split into multiple tokens by the BERT tokenizer. (a) Ground truth (\textbf{negative} sentiment). Note that the BERT prediction on this sample was \textbf{accurate}. (b) Our method for the ground truth [GT] class. (c) Our method for the counter-factual [CF] class. (d) Partial LRP for the GT class. (e) Partial LRP for the CF class. (f) GradCAM for the GT class. (g) GradCAM for the CF class. (h) LRP for the GT class. (i) LRP for the CF class. (j) raw-attention. (k) rollout. }
    \label{fig:my_label1}
\end{figure}

\clearpage
\captionsetup[figure]{font=scriptsize,labelfont=scriptsize}
\begin{figure}[h]
    \centering
    \includegraphics[width=0.99\linewidth]{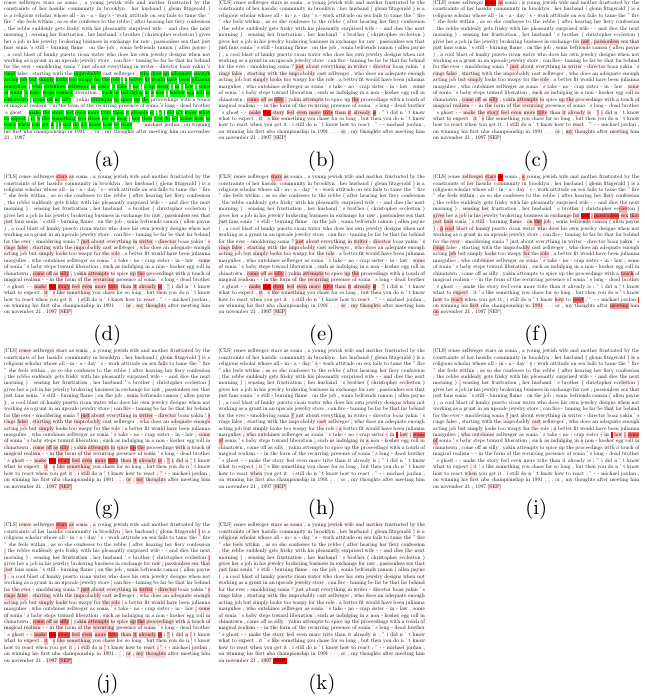}
    \caption{A visualization of the results. For methods that rely on a specific class for propagation, we present both the ground-truth and counter-factual results. The rollout method, as well as the raw attention methods, are class-agnostic. Some words are split into multiple tokens by the BERT tokenizer. (a) Ground truth (\textbf{negative} sentiment). Note that the BERT prediction on this sample was \textbf{accurate}. (b) Our method for the ground truth [GT] class. (c) Our method for the counter-factual [CF] class. (d) Partial LRP for the GT class. (e) Partial LRP for the CF class. (f) GradCAM for the GT class. (g) GradCAM for the CF class. (h) LRP for the GT class. (i) LRP for the CF class. (j) raw-attention. (k) rollout. }
    \label{fig:my_label}
\end{figure}

\clearpage
\captionsetup[figure]{font=scriptsize,labelfont=scriptsize}
\begin{figure}[h]
    \centering
    \includegraphics[width=0.99\linewidth]{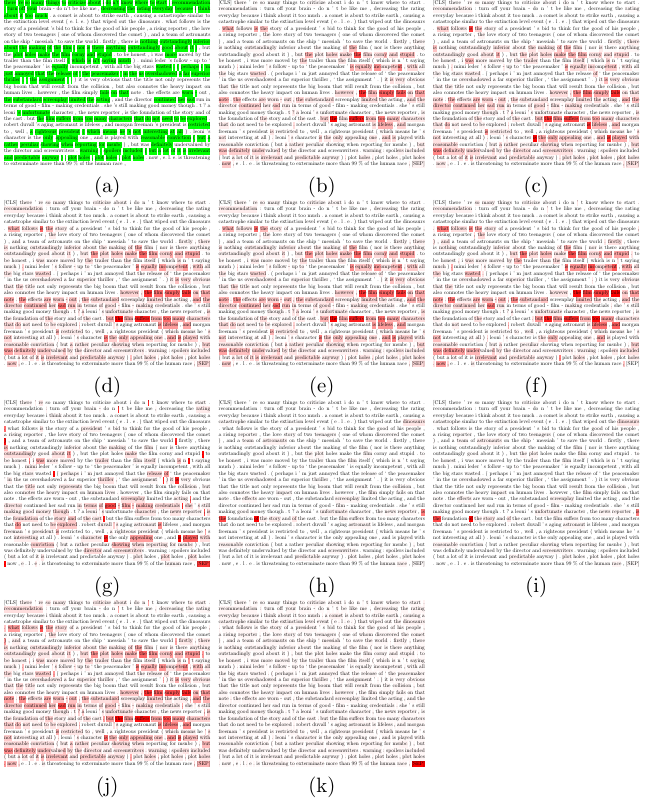}
    \caption{A visualization of the results. For methods that rely on a specific class for propagation, we present both the ground-truth and counter-factual results. The rollout method, as well as the raw attention methods, are class-agnostic. Some words are split into multiple tokens by the BERT tokenizer. (a) Ground truth (\textbf{negative} sentiment). Note that the BERT prediction on this sample was \textbf{accurate}. (b) Our method for the ground truth [GT] class. (c) Our method for the counter-factual [CF] class. (d) Partial LRP for the GT class. (e) Partial LRP for the CF class. (f) GradCAM for the GT class. (g) GradCAM for the CF class. (h) LRP for the GT class. (i) LRP for the CF class. (j) raw-attention. (k) rollout. }
    \label{fig:my_label}
\end{figure}

\clearpage
\captionsetup[figure]{font=scriptsize,labelfont=scriptsize}
\begin{figure}[h]
    \centering
    \includegraphics[width=0.99\linewidth]{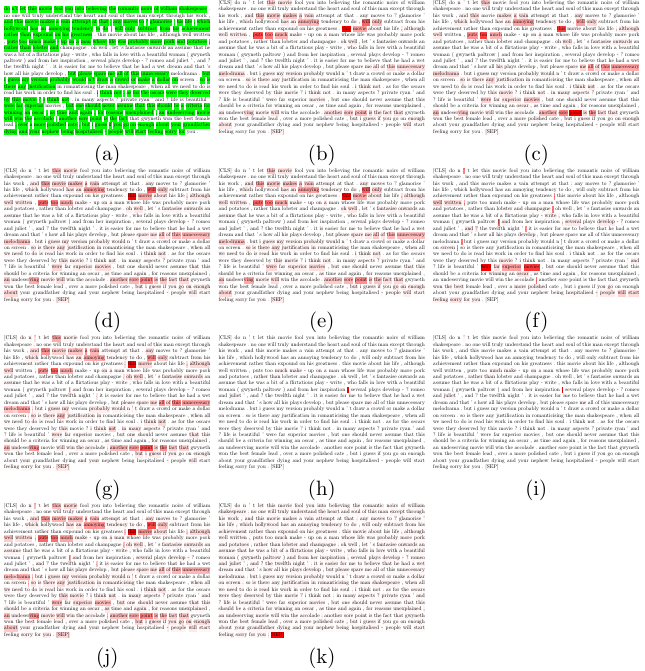}
    \caption{A visualization of the results. For methods that rely on a specific class for propagation, we present both the ground-truth and counter-factual results. The rollout method, as well as the raw attention methods, are class-agnostic. Some words are split into multiple tokens by the BERT tokenizer. (a) Ground truth (negative sentiment). Note that the BERT prediction on this sample was accurate. (b) Our method for the ground truth [GT] class. (c) Our method for the counter-factual [CF] class. (d) Partial LRP for the GT class. (e) Partial LRP for the CF class. (f) GradCAM for the GT class. (g) GradCAM for the CF class. (h) LRP for the GT class. (i) LRP for the CF class. (j) raw-attention. (k) rollout. }
    \label{fig:my_label}
\end{figure}

\clearpage
\captionsetup[figure]{font=scriptsize,labelfont=scriptsize}
\begin{figure}[h]
    \centering
    \includegraphics[width=0.99\linewidth]{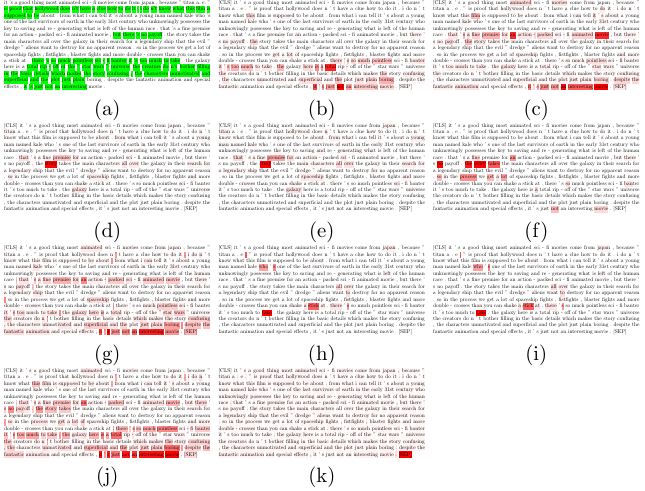}
    \caption{A visualization of the results. For methods that rely on a specific class for propagation, we present both the ground-truth and counter-factual results. The rollout method, as well as the raw attention methods, are class-agnostic. Some words are split into multiple tokens by the BERT tokenizer. (a) Ground truth (\textbf{negative} sentiment). Note that the BERT prediction on this sample was \textbf{accurate}. (b) Our method for the ground truth [GT] class. (c) Our method for the counter-factual [CF] class. (d) Partial LRP for the GT class. (e) Partial LRP for the CF class. (f) GradCAM for the GT class. (g) GradCAM for the CF class. (h) LRP for the GT class. (i) LRP for the CF class. (j) raw-attention. (k) rollout. }
    \label{fig:my_label}
\end{figure}

\clearpage
\captionsetup[figure]{font=scriptsize,labelfont=scriptsize}
\begin{figure}[h]
    \centering
    \includegraphics[width=0.99\linewidth]{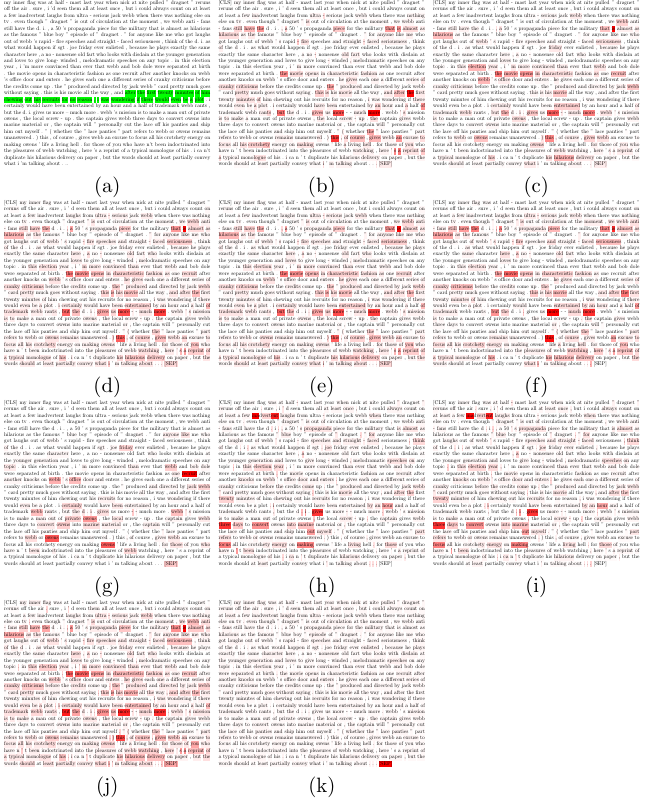}
    \caption{A visualization of the results. For methods that rely on a specific class for propagation, we present both the ground-truth and counter-factual results. The rollout method, as well as the raw attention methods, are class-agnostic. Some words are split into multiple tokens by the BERT tokenizer. (a) Ground truth (\textbf{negative} sentiment). Note that the BERT prediction on this sample was \textbf{mistaken}. (b) Our method for the ground truth [GT] class. (c) Our method for the counter-factual [CF] class. (d) Partial LRP for the GT class. (e) Partial LRP for the CF class. (f) GradCAM for the GT class. (g) GradCAM for the CF class. (h) LRP for the GT class. (i) LRP for the CF class. (j) raw-attention. (k) rollout. }
    \label{fig:my_label}
\end{figure}

\clearpage
\captionsetup[figure]{font=scriptsize,labelfont=scriptsize}

\begin{figure}[h]
    \centering
    \includegraphics[width=0.99\linewidth]{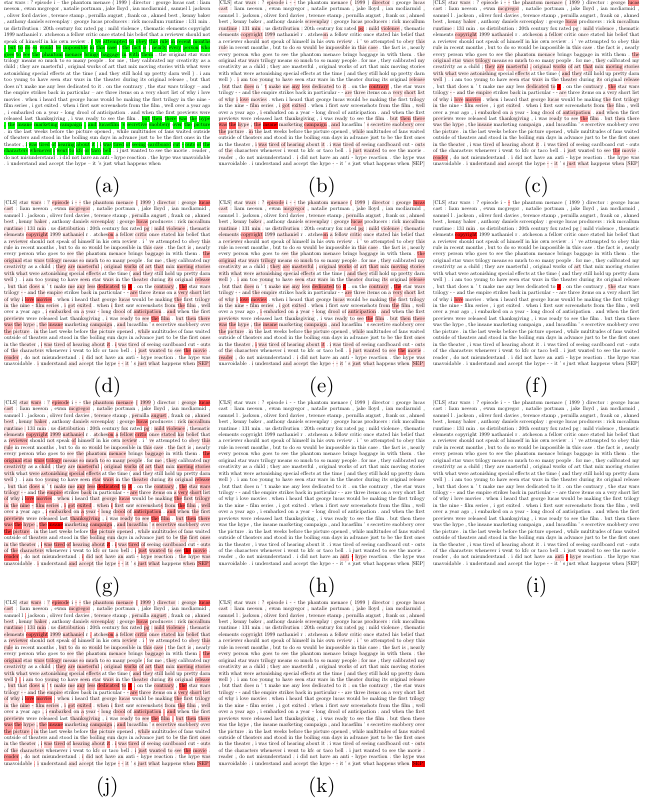}
    \caption{A visualization of the results. For methods that rely on a specific class for propagation, we present both the ground-truth and counter-factual results. The rollout method, as well as the raw attention methods, are class-agnostic. Some words are split into multiple tokens by the BERT tokenizer. (a) Ground truth (\textbf{negative} sentiment). Note that the BERT prediction on this sample was \textbf{mistaken}. (b) Our method for the ground truth [GT] class. (c) Our method for the counter-factual [CF] class. (d) Partial LRP for the GT class. (e) Partial LRP for the CF class. (f) GradCAM for the GT class. (g) GradCAM for the CF class. (h) LRP for the GT class. (i) LRP for the CF class. (j) raw-attention. (k) rollout. }
    \label{fig:my_label}
\end{figure}

\clearpage
\captionsetup[figure]{font=scriptsize,labelfont=scriptsize}
\begin{figure}[h]
    \centering
    \includegraphics[width=0.99\linewidth]{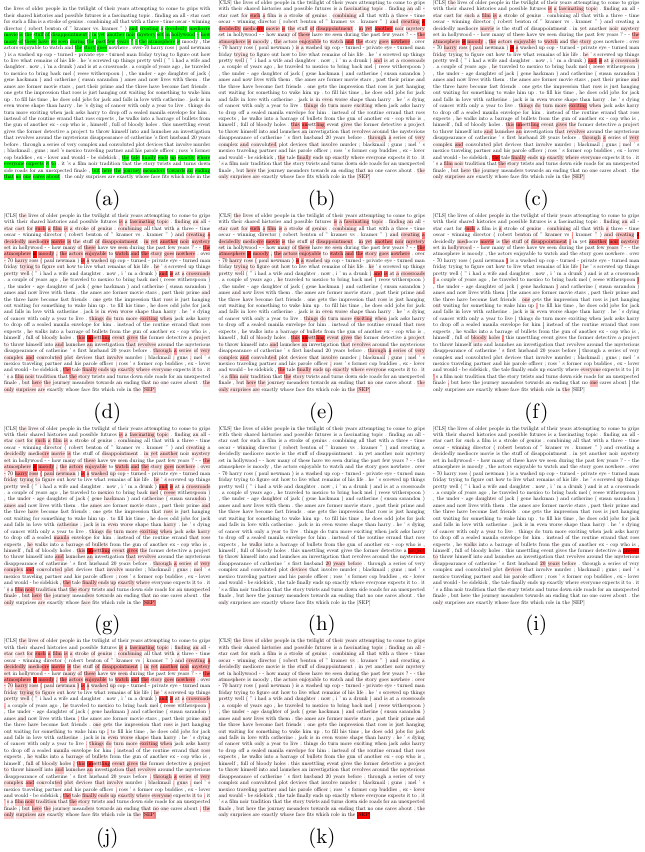}
    \caption{A visualization of the results. For methods that rely on a specific class for propagation, we present both the ground-truth and counter-factual results. The rollout method, as well as the raw attention methods, are class-agnostic. Some words are split into multiple tokens by the BERT tokenizer. (a) Ground truth (\textbf{negative} sentiment). Note that the BERT prediction on this sample was \textbf{mistaken}. (b) Our method for the ground truth [GT] class. (c) Our method for the counter-factual [CF] class. (d) Partial LRP for the GT class. (e) Partial LRP for the CF class. (f) GradCAM for the GT class. (g) GradCAM for the CF class. (h) LRP for the GT class. (i) LRP for the CF class. (j) raw-attention. (k) rollout. }
    \label{fig:my_label}
\end{figure}

\clearpage
\captionsetup[figure]{font=scriptsize,labelfont=scriptsize}
\begin{figure}[h]
    \centering
    \includegraphics[width=0.99\linewidth]{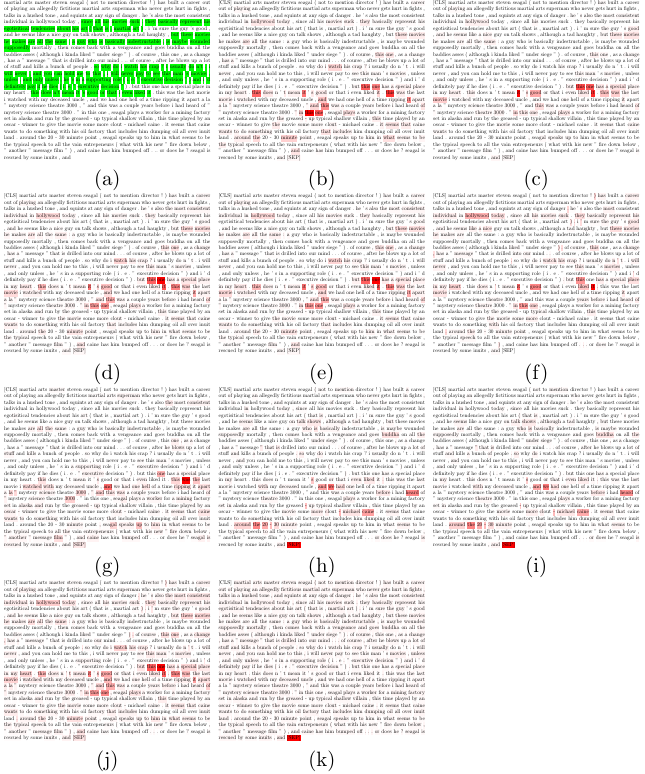}
    \caption{A visualization of the results. For methods that rely on a specific class for propagation, we present both the ground-truth and counter-factual results. The rollout method, as well as the raw attention methods, are class-agnostic. Some words are split into multiple tokens by the BERT tokenizer. (a) Ground truth (\textbf{negative} sentiment). Note that the BERT prediction on this sample was \textbf{mistaken}. (b) Our method for the ground truth [GT] class. (c) Our method for the counter-factual [CF] class. (d) Partial LRP for the GT class. (e) Partial LRP for the CF class. (f) GradCAM for the GT class. (g) GradCAM for the CF class. (h) LRP for the GT class. (i) LRP for the CF class. (j) raw-attention. (k) rollout. }
    \label{fig:my_label}
\end{figure}

\clearpage
\captionsetup[figure]{font=scriptsize,labelfont=scriptsize}
\begin{figure}[h]
    \centering
    \includegraphics[width=0.99\linewidth]{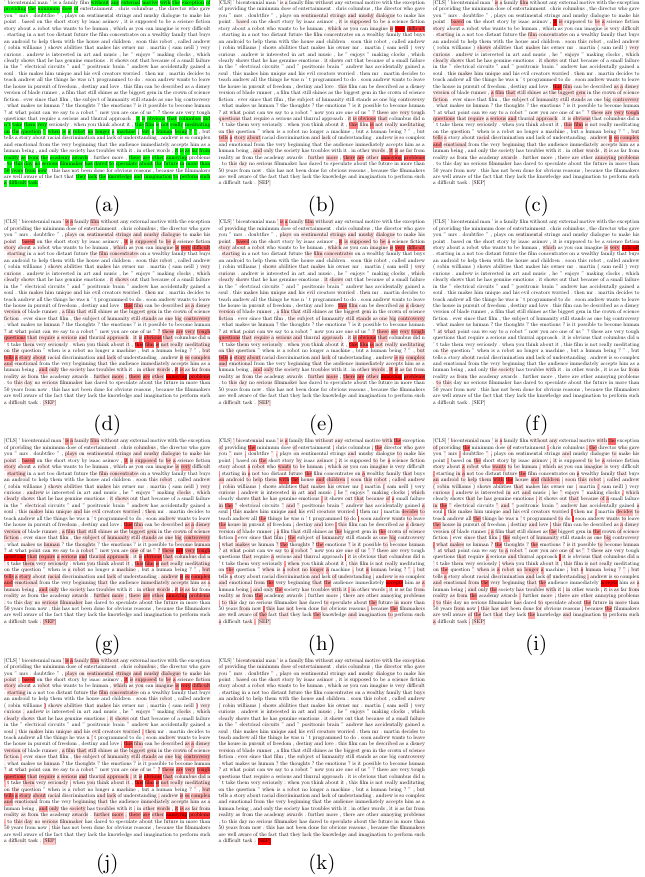}
    \caption{A visualization of the results. For methods that rely on a specific class for propagation, we present both the ground-truth and counter-factual results. The rollout method, as well as the raw attention methods, are class-agnostic. Some words are split into multiple tokens by the BERT tokenizer. (a) Ground truth (\textbf{negative} sentiment). Note that the BERT prediction on this sample was \textbf{mistaken}. (b) Our method for the ground truth [GT] class. (c) Our method for the counter-factual [CF] class. (d) Partial LRP for the GT class. (e) Partial LRP for the CF class. (f) GradCAM for the GT class. (g) GradCAM for the CF class. (h) LRP for the GT class. (i) LRP for the CF class. (j) raw-attention. (k) rollout. }
    \label{fig:my_label}
\end{figure}

\clearpage
\captionsetup[figure]{font=scriptsize,labelfont=scriptsize}
\begin{figure}[h]
    \centering
    \includegraphics[width=0.99\linewidth]{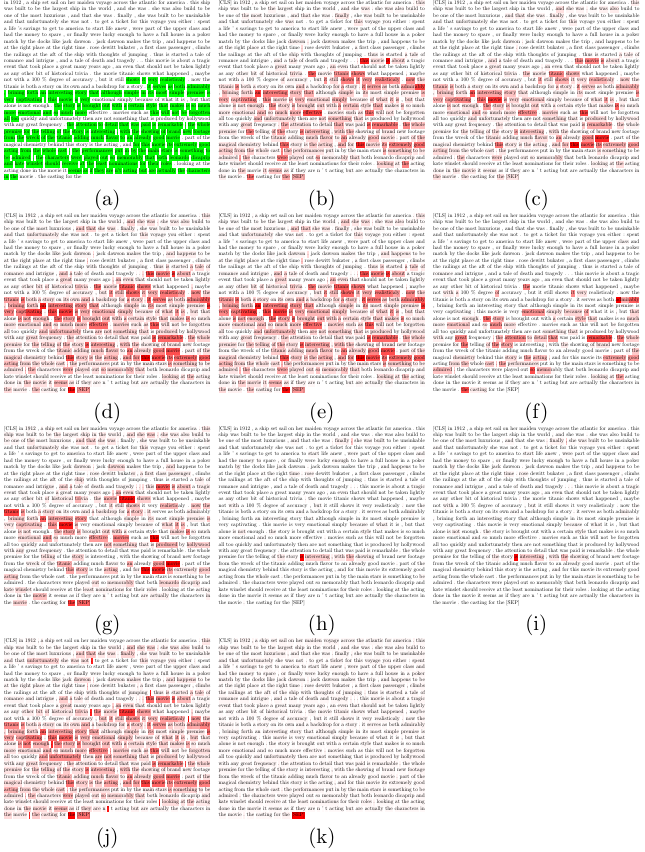}
    \caption{A visualization of the results. For methods that rely on a specific class for propagation, we present both the ground-truth and counter-factual results. The rollout method, as well as the raw attention methods, are class-agnostic. Some words are split into multiple tokens by the BERT tokenizer. (a) Ground truth (\textbf{positive} sentiment). Note that the BERT prediction on this sample was \textbf{accurate}. (b) Our method for the ground truth [GT] class. (c) Our method for the counter-factual [CF] class. (d) Partial LRP for the GT class. (e) Partial LRP for the CF class. (f) GradCAM for the GT class. (g) GradCAM for the CF class. (h) LRP for the GT class. (i) LRP for the CF class. (j) raw-attention. (k) rollout. }
    \label{fig:my_label}
\end{figure}

\clearpage
\captionsetup[figure]{font=scriptsize,labelfont=scriptsize}
\begin{figure}[h]
    \centering
    \includegraphics[width=0.99\linewidth]{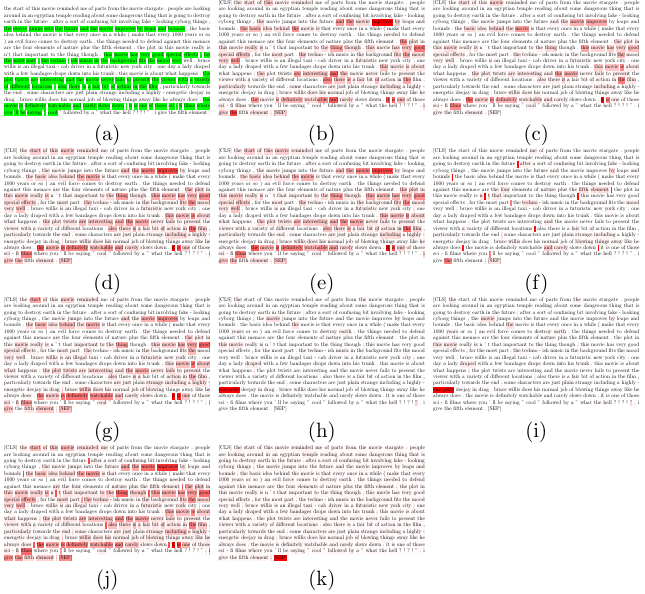}
    \caption{A visualization of the results. For methods that rely on a specific class for propagation, we present both the ground-truth and counter-factual results. The rollout method, as well as the raw attention methods, are class-agnostic. Some words are split into multiple tokens by the BERT tokenizer. (a) Ground truth (\textbf{positive} sentiment). Note that the BERT prediction on this sample was \textbf{accurate}. (b) Our method for the ground truth [GT] class. (c) Our method for the counter-factual [CF] class. (d) Partial LRP for the GT class. (e) Partial LRP for the CF class. (f) GradCAM for the GT class. (g) GradCAM for the CF class. (h) LRP for the GT class. (i) LRP for the CF class. (j) raw-attention. (k) rollout. }
    \label{fig:my_label}
\end{figure}

\clearpage
\captionsetup[figure]{font=scriptsize,labelfont=scriptsize}
\begin{figure}[h]
    \centering
    \includegraphics[width=0.99\linewidth]{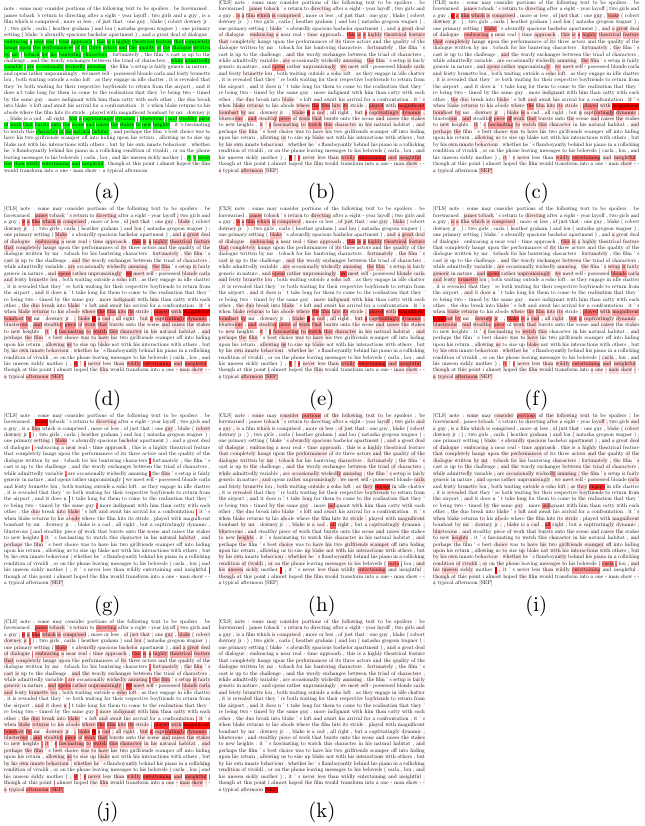}
    \caption{A visualization of the results. For methods that rely on a specific class for propagation, we present both the ground-truth and counter-factual results. The rollout method, as well as the raw attention methods, are class-agnostic. Some words are split into multiple tokens by the BERT tokenizer. (a) Ground truth (\textbf{positive} sentiment). Note that the BERT prediction on this sample was \textbf{accurate}. (b) Our method for the ground truth [GT] class. (c) Our method for the counter-factual [CF] class. (d) Partial LRP for the GT class. (e) Partial LRP for the CF class. (f) GradCAM for the GT class. (g) GradCAM for the CF class. (h) LRP for the GT class. (i) LRP for the CF class. (j) raw-attention. (k) rollout. }
    \label{fig:my_label}
\end{figure}

\clearpage
\captionsetup[figure]{font=scriptsize,labelfont=scriptsize}
\begin{figure}[h]
    \centering
    \includegraphics[width=0.99\linewidth]{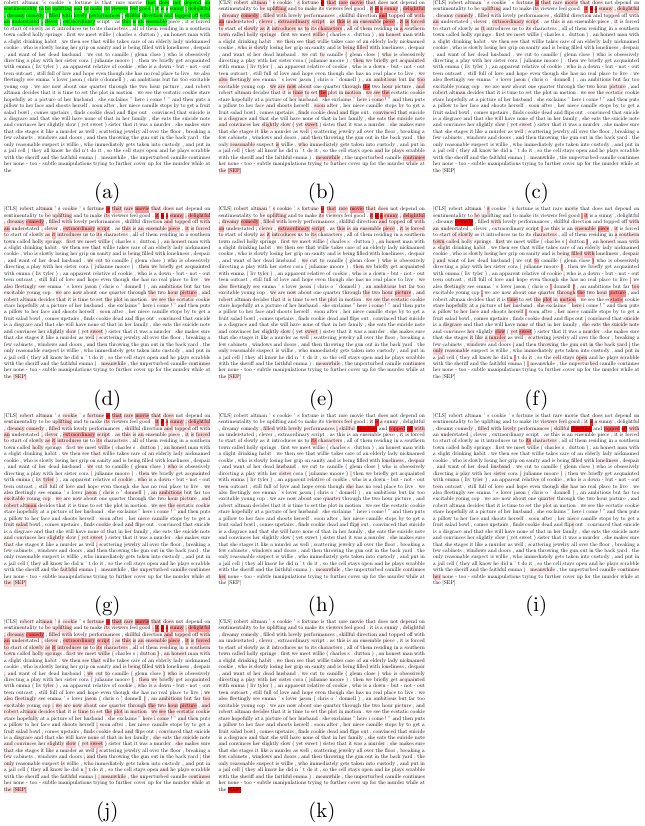}
    \caption{A visualization of the results. For methods that rely on a specific class for propagation, we present both the ground-truth and counter-factual results. The rollout method, as well as the raw attention methods, are class-agnostic. Some words are split into multiple tokens by the BERT tokenizer. (a) Ground truth (\textbf{positive} sentiment). Note that the BERT prediction on this sample was \textbf{accurate}. (b) Our method for the ground truth [GT] class. (c) Our method for the counter-factual [CF] class. (d) Partial LRP for the GT class. (e) Partial LRP for the CF class. (f) GradCAM for the GT class. (g) GradCAM for the CF class. (h) LRP for the GT class. (i) LRP for the CF class. (j) raw-attention. (k) rollout. }
    \label{fig:my_label}
\end{figure}

\clearpage
\captionsetup[figure]{font=scriptsize,labelfont=scriptsize}
\begin{figure}[h]
    \centering
    \includegraphics[width=0.99\linewidth]{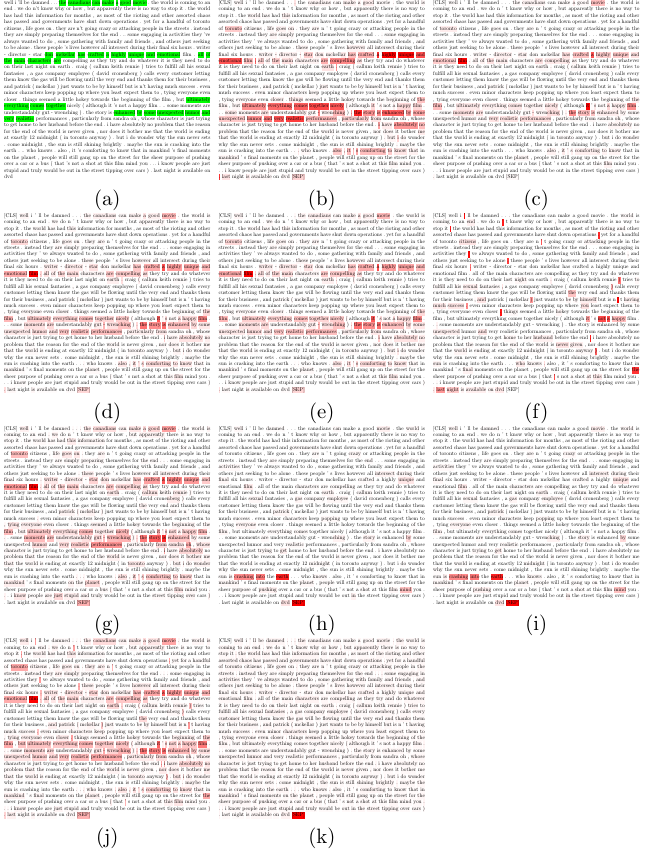}
    \caption{A visualization of the results. For methods that rely on a specific class for propagation, we present both the ground-truth and counter-factual results. The rollout method, as well as the raw attention methods, are class-agnostic. Some words are split into multiple tokens by the BERT tokenizer. (a) Ground truth (\textbf{positive} sentiment). Note that the BERT prediction on this sample was \textbf{accurate}. (b) Our method for the ground truth [GT] class. (c) Our method for the counter-factual [CF] class. (d) Partial LRP for the GT class. (e) Partial LRP for the CF class. (f) GradCAM for the GT class. (g) GradCAM for the CF class. (h) LRP for the GT class. (i) LRP for the CF class. (j) raw-attention. (k) rollout. }
    \label{fig:my_label}
\end{figure}

\clearpage
\captionsetup[figure]{font=scriptsize,labelfont=scriptsize}
\begin{figure}[h]
    \centering
    \includegraphics[width=0.99\linewidth]{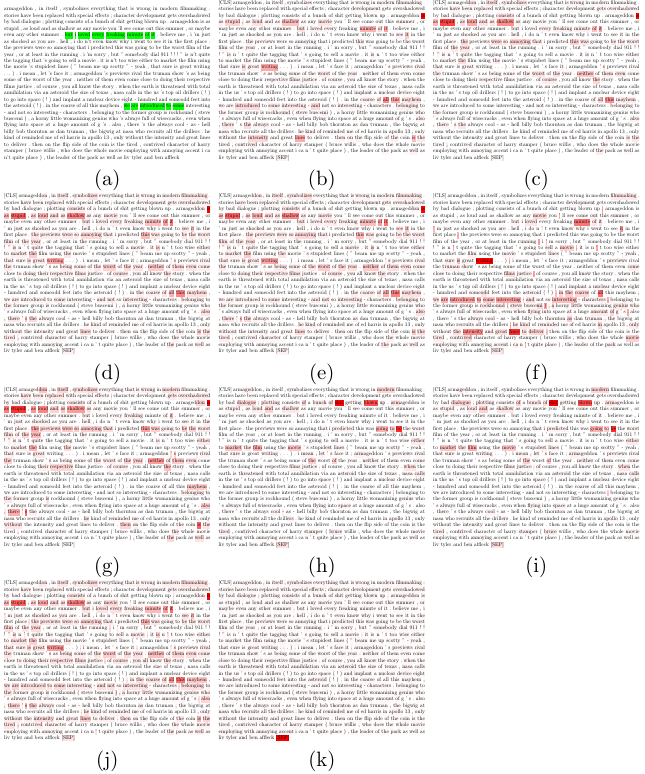}
    \caption{A visualization of the results. For methods that rely on a specific class for propagation, we present both the ground-truth and counter-factual results. The rollout method, as well as the raw attention methods, are class-agnostic. Some words are split into multiple tokens by the BERT tokenizer. (a) Ground truth (\textbf{positive} sentiment). Note that the BERT prediction on this sample was \textbf{mistaken}. (b) Our method for the ground truth [GT] class. (c) Our method for the counter-factual [CF] class. (d) Partial LRP for the GT class. (e) Partial LRP for the CF class. (f) GradCAM for the GT class. (g) GradCAM for the CF class. (h) LRP for the GT class. (i) LRP for the CF class. (j) raw-attention. (k) rollout. }
    \label{fig:my_label}
\end{figure}

\clearpage
\captionsetup[figure]{font=scriptsize,labelfont=scriptsize}
\begin{figure}[h]
    \centering
    \includegraphics[width=0.99\linewidth]{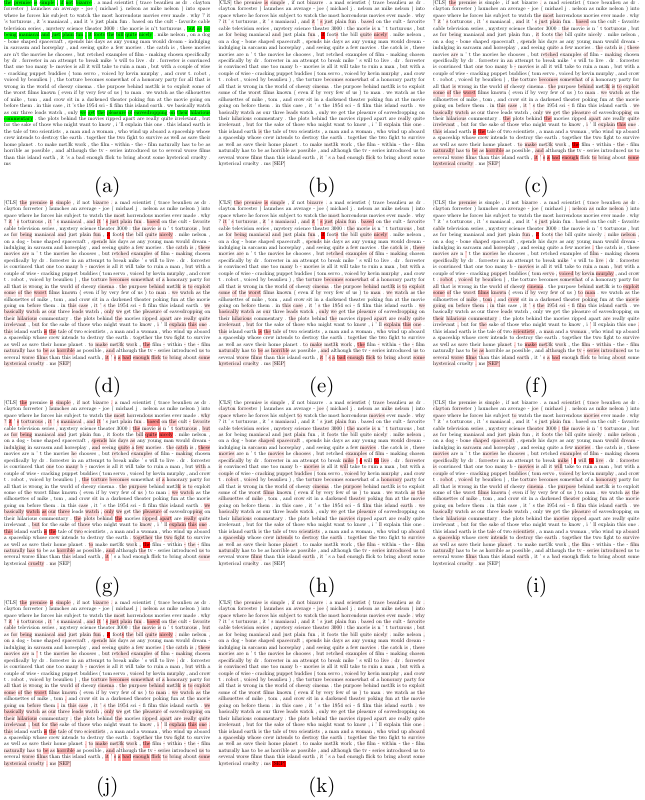}
    \caption{A visualization of the results. For methods that rely on a specific class for propagation, we present both the ground-truth and counter-factual results. The rollout method, as well as the raw attention methods, are class-agnostic. Some words are split into multiple tokens by the BERT tokenizer. (a) Ground truth (\textbf{positive} sentiment). Note that the BERT prediction on this sample was \textbf{mistaken}. (b) Our method for the ground truth [GT] class. (c) Our method for the counter-factual [CF] class. (d) Partial LRP for the GT class. (e) Partial LRP for the CF class. (f) GradCAM for the GT class. (g) GradCAM for the CF class. (h) LRP for the GT class. (i) LRP for the CF class. (j) raw-attention. (k) rollout. }
    \label{fig:my_label}
\end{figure}

\clearpage
\captionsetup[figure]{font=scriptsize,labelfont=scriptsize}
\begin{figure}[h]
    \centering
    \includegraphics[width=0.99\linewidth]{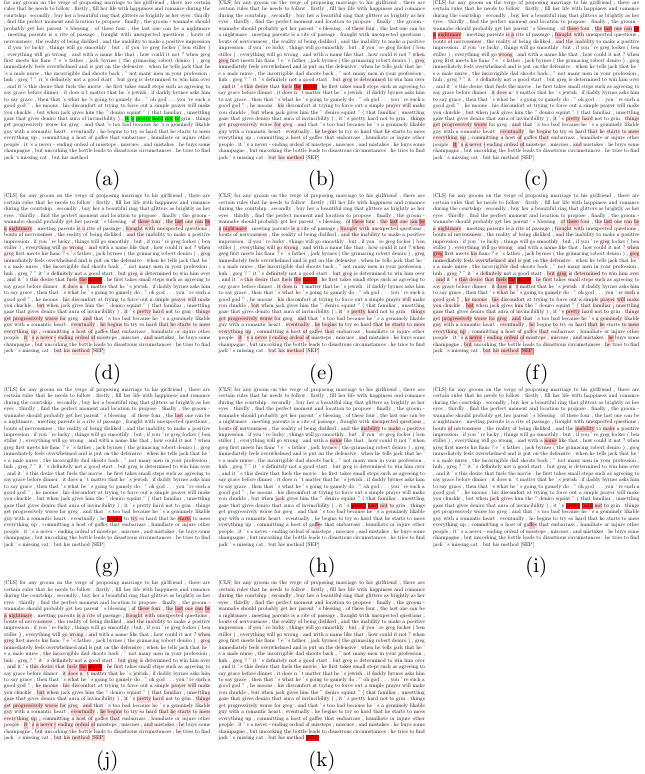}
    \caption{A visualization of the results. For methods that rely on a specific class for propagation, we present both the ground-truth and counter-factual results. The rollout method, as well as the raw attention methods, are class-agnostic. Some words are split into multiple tokens by the BERT tokenizer. (a) Ground truth (\textbf{positive} sentiment). Note that the BERT prediction on this sample was \textbf{mistaken}. (b) Our method for the ground truth [GT] class. (c) Our method for the counter-factual [CF] class. (d) Partial LRP for the GT class. (e) Partial LRP for the CF class. (f) GradCAM for the GT class. (g) GradCAM for the CF class. (h) LRP for the GT class. (i) LRP for the CF class. (j) raw-attention. (k) rollout. }
    \label{fig:my_label}
\end{figure}

\clearpage
\captionsetup[figure]{font=scriptsize,labelfont=scriptsize}
\begin{figure}[h]
    \centering
    \includegraphics[width=0.99\linewidth]{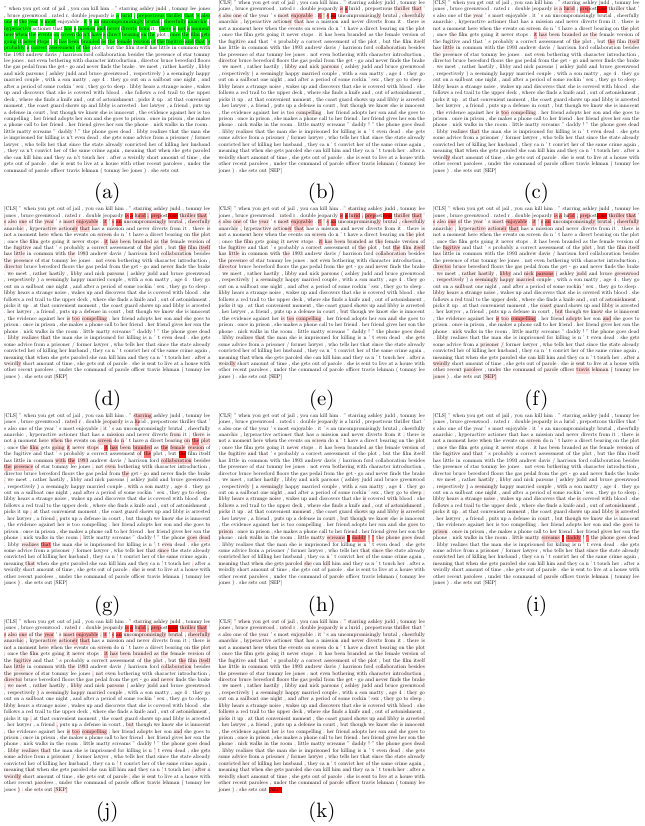}
    \caption{A visualization of the results. For methods that rely on a specific class for propagation, we present both the ground-truth and counter-factual results. The rollout method, as well as the raw attention methods, are class-agnostic. Some words are split into multiple tokens by the BERT tokenizer. (a) Ground truth (\textbf{positive} sentiment). Note that the BERT prediction on this sample was \textbf{mistaken}. (b) Our method for the ground truth [GT] class. (c) Our method for the counter-factual [CF] class. (d) Partial LRP for the GT class. (e) Partial LRP for the CF class. (f) GradCAM for the GT class. (g) GradCAM for the CF class. (h) LRP for the GT class. (i) LRP for the CF class. (j) raw-attention. (k) rollout. }
    \label{fig:my_label}
\end{figure}


\end{document}